\documentclass[lettersize,journal]{IEEEtran}

\IEEEoverridecommandlockouts                              

\usepackage{graphics} 
\usepackage{epsfig} 
\usepackage{times} 
\usepackage{amsmath} 
\usepackage{amssymb}  
\usepackage{amsthm}
\usepackage{hyperref}
\usepackage{multirow}
\usepackage{caption,setspace}
\usepackage{booktabs}
\usepackage{arydshln}
\usepackage{tablefootnote}
\usepackage{csquotes}
\usepackage{footnote}
\makesavenoteenv{tabular}
\makesavenoteenv{table}
\usepackage[]{threeparttable}
\usepackage{breqn}
\usepackage{subcaption}
\usepackage{algorithm, algpseudocode}
\usepackage{float}
\usepackage[dvipsnames]{xcolor}

\usepackage{cite}

\DeclareMathOperator*{\argmin}{arg\,min}
\newtheorem{definition}{Definition}
\newtheorem{remark}{Remark}
\newtheorem{lemma}{Lemma}
\newtheorem{theorem}{Theorem}

\makeatletter
\def\adl@drawiv#1#2#3{%
        \hskip.5\tabcolsep
        \xleaders#3{#2.5\@tempdimb #1{1}#2.5\@tempdimb}%
                #2\z@ plus1fil minus1fil\relax
        \hskip.5\tabcolsep}
\newcommand{\cdashlinelr}[1]{%
  \noalign{\vskip\aboverulesep
           \global\let\@dashdrawstore\adl@draw
           \global\let\adl@draw\adl@drawiv}
  \cdashline{#1}
  \noalign{\global\let\adl@draw\@dashdrawstore
           \vskip\belowrulesep}}
\makeatother

\title{\LARGE \bf
Fast Ergodic Search With Kernel Functions 
}

\author{Max Muchen Sun, Ayush Gaggar, Pete Trautman, and Todd Murphey
\thanks{Max Muchen Sun, Ayush Gaggar and Todd Murphey are with the Department of Mechanical Engineering, Northwestern University, Evanston, IL 60208, USA. Email: {\tt\small msun@u.northwestern.edu}}
\thanks{Pete Trautman is with Honda Research Institute, San Jose, CA 95134, USA}
}

\begin{document}
\allowdisplaybreaks

\bstctlcite{IEEEexample:BSTcontrol}

\maketitle

\begin{abstract}
Ergodic search enables optimal exploration of an information distribution with guaranteed asymptotic coverage of the search space. However, current methods typically have exponential computational complexity and are limited to Euclidean space. We introduce a computationally efficient ergodic search method. Our contributions are two-fold: First, we develop a kernel-based ergodic metric, generalizing it from Euclidean space to Lie groups. We prove this metric is consistent with the exact ergodic metric and ensures linear complexity. Second, we derive an iterative optimal control algorithm for trajectory optimization with the kernel metric. Numerical benchmarks show our method is two orders of magnitude faster than the state-of-the-art method. Finally, we demonstrate the proposed algorithm with a peg-in-hole insertion task. We formulate the problem as a coverage task in the space of SE(3) and use a 30-second-long human demonstration as the prior distribution for ergodic coverage. Ergodicity guarantees the asymptotic solution of the peg-in-hole problem so long as the solution resides within the prior information distribution, which is seen in the 100\% success rate. 
\end{abstract}

\IEEEpeerreviewmaketitle


\section{Introduction}
\label{sec:introduction}

Robots often need to search an environment driven by a distribution of information of interest. Examples include search-and-rescue based on human-annotated maps or aerial images~\cite{murphy_human-robot_2004}\cite{shah_multidrone_2020}, object tracking under sensory or motion uncertainty~\cite{abraham_decentralized_2018}\cite{shetty_ergodic_2022}, and data collection in active learning~\cite{abraham_active_2019}\cite{ prabhakar_mechanical_2022}. The success of such tasks depends on both the richness of the information representation and the effectiveness of the search algorithm. While advances in machine perception and sensor design have substantially improved the quality of information representation, generating effective search strategies for the given information remains an open challenge. 

Motivated by such a challenge, ergodicity---as an information-theoretic coverage metric---is proposed to optimize search decisions~\cite{mathew_metrics_2011}. Originating in statistical mechanics~\cite{petersen_ergodic_1989}, and more recently the study of fluid mixing~\cite{mathew_multiscale_2005}, the ergodic metric measures the time-averaged behavior of a dynamical system with respect to a spatial distribution---a dynamic system is ergodic with respect to a spatial distribution if the system visits any region of the space for an amount of time proportional to the integrated value of the distribution over the region. Optimizing the ergodic metric guides the robot to cover the whole search space asymptotically while investing more time in areas with higher information values. Recent work has also shown that such a search strategy closely mimics the search behaviors observed across mammal and insect species as a proportional betting strategy for information~\cite{chen_tuning_2020}. 

Despite the theoretical advantages and tight connections to biological systems, current ergodic search methods are not suitable for all robotic tasks. The ergodic metric proposed in~\cite{mathew_metrics_2011} has an exponential computation complexity in the search space dimension~\cite{shetty_ergodic_2022}\cite{sun_scale-invariant_2022}, limiting its applications in spaces with fewer than 3 dimensions. Moreover, common robotic tasks, in particular vision or manipulation-related tasks, often require operations in non-Euclidean spaces, such as the space of rotations or rigid-body transformations. However, the ergodic metric in~\cite{mathew_metrics_2011} is restricted only in the Euclidean space. 

In this article, we propose an alternative formula for ergodic search across Euclidean space and Lie groups with significantly improved computational efficiency. Our formula is based on the difference between target information distribution and the spatial empirical distribution of the trajectory, measured through function space inner product. We re-derive the ergodic metric and show that ergodicity can be computed as the summation of the integrated likelihood of the trajectory within the spatial distribution and the uniformity of the trajectory measured with a kernel function. We name this formula the \emph{kernel ergodic metric} and show that it is asymptotically consistent with the exact ergodic metric in~\cite{mathew_metrics_2011} but has a linear computation complexity in the search space dimension instead of an exponential one. We derive the metric for both Euclidean space and Lie groups. Moreover, we derive an iterative optimal control method for non-linear dynamical systems based on the iterative linear quadratic regulator (iLQR) algorithm~\cite{hauser_projection_2002}. We further generalize the derivations to Lie groups.

We compare the computation efficiency of the proposed algorithm with the state-of-the-art fast ergodic search method~\cite{shetty_ergodic_2022} through a comprehensive benchmark. The proposed method is at least two orders of magnitude faster to reach the same level of ergodicity across 2D to 6D spaces and with first-order and second-order system dynamics. We further demonstrate the proposed algorithm for a peg-in-hole insertion task on a 7 degrees-of-freedom robot arm. We formulate the problem as an ergodic coverage task in the space of SE(3), where the robot needs to simultaneously explore its end-effector’s position and orientation, using a 30-second-long human demonstration as the prior distribution for ergodic coverage. We verify that the asymptotic coverage property of ergodic search leads to the task’s $100\%$ success rate.

\begin{table*}[h]
    \centering
    \captionsetup{justification=centering}
    \caption{Properties of different ergodic search methods.}
    \label{table:property_comparison}
    \setlength{\tabcolsep}{10.5pt}
    \begin{tabular}{cccccc}
        \toprule
        \multirow{3}{*}{\begin{tabular}[c]{@{}c@{}c@{}}Methods\end{tabular}} & \multirow{3}{*}{\begin{tabular}[c]{@{}c@{}}Asymptotic\\Consistency\end{tabular}} &\multirow{3}{*}{\begin{tabular}[c]{@{}c@{}}Real-Time\\Computation\end{tabular}} &\multirow{3}{*}{\begin{tabular}[c]{@{}c@{}}Long\\Planning\\Horizon\end{tabular}} &\multirow{3}{*}{\begin{tabular}[c]{@{}c@{}}Lie Group\\Generalization\end{tabular}} &\multirow{3}{*}{\begin{tabular}[c]{@{}c@{}}Complexity\\w.r.t.\\Space Dimension\end{tabular}} \\
         & & & & & \\
         & & & & & \\
        \midrule
        Mathew et al. \cite{mathew_metrics_2011} & \checkmark & \checkmark &   &   & Exponential \\
        \midrule
        Miller et al. \cite{miller_trajectory_2013-1} & \checkmark &   & \checkmark &   & Exponential \\
        \midrule
        Miller et al. \cite{miller_trajectory_2013} & \checkmark &   & \checkmark & \checkmark & Exponential \\
        \midrule
        Abraham et al. \cite{abraham_ergodic_2021} &   & \checkmark & \checkmark &   & Polynomial to Exponential$^*$ \\
        \midrule
        Shetty et al. \cite{shetty_ergodic_2022} & \checkmark & \checkmark &   & \checkmark & Superlinear \\
        \midrule
        \textbf{Ours} & \checkmark & \checkmark & \checkmark & \checkmark & Linear \\
        \bottomrule 
    \end{tabular} \\
    \vspace{+2pt}
    \footnotesize{$^*$ The method proposed in Abraham et al. \cite{abraham_ergodic_2021} uses Monte-Carlo (MC) integration and has a linear complexity to the number of samples. However, to guarantee a consistent MC integration estimate, the number of samples has a growth rate between polynomial and exponential to the dimension~\cite{tang_note_2023}.}\\
\end{table*}

The rest of the paper is organized as follows: Section~\ref{sec:related_work} discusses related work on ergodic search. Section~\ref{sec:problem_formulation} formulates the ergodic search problem and introduces necessary notations. Section~\ref{sec:ergodic_metric} derives the proposed ergodic metric and a theoretical analysis of its formal properties. Section~\ref{sec:ergodic_control} introduces the theory and algorithm of controlling a non-linear dynamic system to optimize the proposed metric. Section~\ref{sec:lie_groups} generalizes the previous derivations from Euclidean space to Lie group. Section~\ref{sec:evaulation} includes the numerical evaluation and hardware verification of the proposed ergodic search algorithm, followed by a conclusion and further discussion in Section~\ref{sec:conclusion}. The code of our implementation is available at \url{https://sites.google.com/view/kernel-ergodic/}.


\section{Related Works:\\Ergodic Theory and Ergodic Search} \label{sec:related_work}

Ergodic theory studies the connection between the time-averaged and space-averaged behaviors of a dynamical system. Originating in statistical mechanics, it has now expanded to a full branch of mathematics with deep connections to other branches, such as information theory, measure theory, and functional analysis. We refer the readers to \cite{walters_introduction_2000} for a more comprehensive review of the ergodic theory in general. For decision-making, the ergodic theory provides formal principles to reason over decisions based on the time and space-averaged behaviors of the environment or of the agent itself. The application of ergodic theory in robotic search tasks was first introduced in~\cite{mathew_metrics_2011}. In this seminal work, the formal definition of ergodicity in the context of a search task is given as the difference between the time-averaged spatial statistics of the agent's trajectory and the target information distribution to search in. A quantitative measure of such difference is also introduced with the name~\emph{spectral multi-scale coverage} (SMC) metric, as well as a closed-form model predictive controller with infinitesimally small planning horizon for both first-order and second-order linear dynamical systems. We refer to the SMC metric in~\cite{mathew_metrics_2011} as the \emph{Fourier ergodic metric} in the rest of the paper.

Ergodic search has since been applied to generate informative search behaviors in robotic applications, including multi-modal target localization~\cite{mavrommati_real-time_2018}, object detection~\cite{abraham_active_2019}, imitation learning~\cite{kalinowska_ergodic_2021}, robotic assembly~\cite{ehlers_imitating_2019}\cite{shetty_ergodic_2022}, and automated data collection for generative models~\cite{prabhakar_mechanical_2022}. The metric has also been applied to non-search robotic applications, such as point cloud registration~\cite{sun_scale-invariant_2022}. Furthermore, ergodic search has also been extended to better satisfy other requirements from common robotic tasks, such as safety-critical search~\cite{lerch_safety-critical_2023}, multi-objective search~\cite{ren_pareto-optimal_2023}, and time optimal search~\cite{dong_time_2023}. 

There are several limitations of the Fourier ergodic search framework from~\cite{mathew_metrics_2011}: (1) the controller is limited with an infinitesimally small planning horizon. Thus, it often requires an impractically long exploration period to generate good coverage; (2) it is costly to scale the Fourier ergodic metric to higher dimension spaces; (3) it is non-trivial to generalize the metric to non-Euclidean spaces. Previous works have designed controllers to optimize the trajectory over a longer horizon. A trajectory optimization method was introduced in \cite{miller_trajectory_2013}, which optimizes the Fourier ergodic metric iteratively for a nonlinear system by solving a linear-quadratic regulator (LQR) problem in each iteration. A model predictive control method based on hybrid systems theory was introduced in \cite{mavrommati_real-time_2018}, which is later extended to support decentralized multi-agent ergodic search in~\cite{abraham_decentralized_2018}. However, since these methods optimize the Fourier ergodic metric, they are still limited by the computation cost of evaluating the metric itself. In~\cite{abraham_ergodic_2021}, an approximated ergodic search framework is proposed. The empirical spatial distribution of the robot trajectory is approximated as a Gaussian-mixture model, and the Fourier ergodic metric is replaced with the Kullback-Leibler (KL) divergence between the Gaussian-mixture distribution and target information distribution, estimated using Monte-Carlo (MC) integration. While this framework has a linear time complexity to the number of samples used for MC integration, to guarantee a consistent estimate of the KL divergence, the number of samples has a growth rate varying between polynomial and exponential to the search space dimension~\cite{tang_note_2023}. A new computation scheme was introduced in \cite{shetty_ergodic_2022} to accelerate the evaluation of the Fourier ergodic metric using the technique of tensor train decomposition. This framework is demonstrated on an ergodic search task in a 6-dimensional space. However, this framework is limited to an infinitesimally small planning horizon, and even though the tensor train technique significantly improves the scalability of the Fourier ergodic metric, the computational cost is still expensive for planning with longer horizons. As for extending the ergodic search to non-Euclidean spaces, an extension to the special Euclidean group SE(2) was introduced in \cite{miller_trajectory_2013} by defining the Fourier basis function on SE(2). However, defining the Fourier basis function for other Lie groups is non-trivial, and the method has the same computation complexity as in Euclidean space. The tensor train framework from~\cite{shetty_ergodic_2022} can also be generalized to Lie groups. However, the generalization is for the controller instead of the metric; thus, it is limited to an infinitesimally small planning horizon. Our proposed ergodic search framework is built on top of a scalable ergodic metric that is asymptotically consistent with the exact ergodic metric and the Fourier ergodic metric in~\cite{mathew_metrics_2011}, alongside rigorous generalization to Lie groups. A comparison of the properties of different ergodic search methods is shown in Table~\ref{table:property_comparison}.

\section{Preliminaries} \label{sec:problem_formulation}

\subsection{Notations and Definitions}
We denote the state of the robot as $s\in\mathcal{S}$, where $\mathcal{S}$ is a bounded set within an $n$-dimensional Euclidean space. Later in the paper, we will extend the state of the robot to Lie groups. We assume the robot's motion is governed by the following dynamics:
\begin{align}
    \dot{s}(t) = f(s(t), u(t)), \label{eq:robot_dynamics}
\end{align} where $u(t)\in\mathcal{U}\subset\mathbb{R}^m$ is the control signal. The dynamics function $f(\cdot,\cdot)$ is differentiable with respect to both $s(t)$ and $u(t)$. We denote a probability density function defined over the bounded state space $\mathcal{S}$ as $p(x):\mathcal{S}\mapsto\mathbb{R}_0^+$, which must satisfy: 
\begin{align}
    \int_{\mathcal{S}} p(x)dx=1 \text{ and } p(x)\geq 0\quad\forall x\in\mathcal{S}. \label{eq:distr_constraints}
\end{align} We define a trajectory $s(t):[0,T]\mapsto\mathcal{S}$ as a continuous mapping from time to a state in the bounded state space.

\begin{definition}[Inner product]
    The inner product $\langle\cdot ,\cdot\rangle$ between functions, similar to its finite-dimensional counterpart in the vector space, is defined as:
    \begin{align}
        \langle f(x), g(x) \rangle = \int f(x) g(x) dx\label{eq:inner_product_integral}.
    \end{align}
\end{definition}

\begin{definition}[Dirac delta function]
    The Dirac delta function $\delta(x)$ is the limit of a sequence of functions that satisfy:
    \begin{gather}
        \delta(x{-}s) = \lim_{\epsilon\rightarrow 0^+} \delta_{\epsilon}(x{-}s) \quad \text{s.t.} \label{eq:delta_def} \\
         \langle \delta(x{-}s), f(x) \rangle = \lim_{\epsilon\rightarrow 0^+} \int \delta_{\epsilon}(x{-}s) f(x) dx = f(s), \text{ } \forall s \in \mathcal{X} , \label{eq:delta_inner_product}
    \end{gather} where $\delta_{\epsilon}(\cdot)$ is sometimes called a nascent delta function~\cite{mack_fundamental_2007}.
\end{definition}

\begin{remark} \label{remark:delta_function}
    The Dirac delta function is not a conventional function defined as a point-wise mapping, instead it is a \emph{generalized function} (also called a distribution) defined based on its inner product property shown in (\ref{eq:delta_inner_product}). We refer the readers to \cite{lighthill_introduction_1958}, \cite{rudin_functional_1991}, and \cite{strichartz_guide_1994} for more information regarding the Dirac delta function and generalized functions. 
\end{remark}

\begin{lemma} \label{lemma:delta_inner_product_itself}
    The inner product between two Dirac delta functions is (see Appendix II of~\cite{cohen-tannoudji_quantum_1977} for detailed derivation):
    \begin{align}
        \langle \delta(x{-}s_1), \delta(x{-}s_2) \rangle = \int \delta(x{-}s_1) \delta(x{-}s_2) dx = \delta(s_1{-}s_2) . \nonumber
    \end{align}
\end{lemma}

\begin{definition}[Trajectory empirical distribution]
Given a trajectory $s(t):[0, T]\mapsto\mathcal{S}$ of the robot, we define the empirical distribution of the trajectory as:
\begin{align}
    c_{s}(x) = \frac{1}{T} \int_{0}^{T} \delta(x-s(t)) dt. \label{eq:trajectory_spatial_statistics}
\end{align}
\end{definition}

\begin{lemma} \label{lemma:empirical_inner_product}
    The inner product between $c_s(x)$ and another function $f(x)$ is:
    \begin{align}
        \langle c_s(x), f(x) \rangle & = \int \left( \frac{1}{T} \int_0^T \delta(x-s(t)) dt \right) f(x) dx \nonumber \\
        & = \frac{1}{T} \int_0^T \left( \int \delta(x-s(t)) f(x) dx \right) dt \nonumber \\
        & = \frac{1}{T} \int_0^T f(s(t)) dt .
    \end{align}
\end{lemma}

\begin{lemma} \label{lemma:empirical_norm}
    The inner product $\langle c_s(x){,} c_s(x) \rangle$ is:
    \begin{align}
        & \int \left( \frac{1}{T} \int_0^T \delta(x-s(t_1)) dt_1 \right) \left( \frac{1}{T} \int_0^T \delta(x-s(t_2)) dt_2 \right) dx \nonumber \\
        & = \int \left( \frac{1}{T^2} \int_0^T \int_0^T \delta(x-s(t_1)) \delta(x-s(t_2)) dt_1 dt_2 \right) dx \nonumber \\
        & = \frac{1}{T^2} \int_0^T \int_0^T \left( \int \delta(x-s(t_1)) \delta(x-s(t_2)) dx \right) dt_1 dt_2 \nonumber \\
        & = \frac{1}{T^2} \int_0^T \int_0^T \delta(s(t_1)-s(t_2)) dt_1 dt_2 . \label{eq:emp_inner}
    \end{align}
\end{lemma}

\subsection{Ergodicity and the exact ergodic metric}

The definition of ergodicity states that a dynamic system is ergodic with respect to the distribution if and only if \emph{the system visits any region of the space for an amount of time proportional to the integrated value of the distribution over the region}~\cite{mathew_metrics_2011}. An exact metric of ergodicity is then introduced, which we name the \emph{exact ergodic metric}.

\begin{definition}[Exact ergodic metric] The exact ergodic metric between a dynamic system $s(t)$ and a spatial distribution $p(x)$ is defined as follow~\cite{mathew_metrics_2011}:
\begin{align}
    & \mathcal{E}(s(t), p(x)) = \nonumber \\
    & \int_0^R \int_{\mathcal{S}} \left[ 
\frac{1}{T}\int_0^T \mathbf{1}_{(x,r)}(s(\tau)) d\tau {-} \int_{\mathcal{S}} \mathbf{1}_{(x,r)}(y) p(y) dy \right]^2 dx dr , \label{eq:exact_metric}
\end{align} where $\mathbf{1}_{(x,r)}$ is a spherical indicator function centered at $x$ with a radius of $r$:
\begin{align}
    \mathbf{1}_{(x,r)}(s) = \begin{cases}
        1, \quad \text{if $\Vert x{-}s \Vert \leq r$} \\
        0, \quad \text{otherwise.}
    \end{cases}
\end{align} If the system $s(t)$ is ergodic, then the following limit holds:
\begin{align}
    \lim_{T\rightarrow\infty} \mathcal{E}(s(t), p(x)) = 0.
\end{align}
\end{definition}

\begin{lemma}[Asymptotic coverage]
    Based on the definition of the exact ergodic metric (\ref{eq:exact_metric}), if the spatial distribution $p(x)$ has a non-zero density at any point of the state space, an ergodic system, while systematically spending more time over regions with higher information density and less time over regions with less information density, will eventually cover any state in the state space with the exploration time approaching infinity.
\end{lemma}

Despite the asymptotic coverage property, calculating the metric and optimizing a trajectory with respect to it is infeasible in practice. This infeasibility motivates the research of ergodic control, including our work, to develop approximations of the \emph{exact} ergodic metric that are efficient to calculate and optimize in practice while preserving the non-myopic coverage behavior of an ergodic system. Below, we will review one of the most commonly used approximated ergodic metrics in practice, the Fourier ergodic metric.

\subsection{Fourier ergodic metric}

Motivated by the need for a practical measure of ergodicity for robotic search tasks, the Fourier ergodic metric was proposed in~\cite{mathew_metrics_2011}. We now briefly introduce the formula of the Fourier ergodic metric.

\begin{definition}[Fourier basis function]
    The Fourier ergodic metric assumes the robot operates in a $n$-dimensional rectangular Euclidean space, denoted as $\mathcal{S}=[0, L_1]\times\cdots\times[0, L_n]$. The Fourier basis function $f_{\mathbf{k}}(x) : \mathcal{S} \mapsto \mathbb{R}$ is defined as:
    \begin{align}
         f_{\mathbf{k}}(x) = \frac{1}{h_{\mathbf{k}}}\prod_{i=1}^{n} \cos\left( \frac{k_i\pi}{L_i} x_i \right), \label{eq:fourier_basis}
    \end{align} where 
    \begin{align}
        & x = [x_1, x_2, \cdots, x_n] \in \mathcal{S} \nonumber \\
        & \mathbf{k} = [k_1, \cdots, k_n] \in \mathcal{K} \subset \mathbb{N}^n \nonumber \\
        & \mathcal{K} = [0, 1, \cdots, K_1]\times\cdots\times [0, 1, \cdots, K_n], \nonumber
    \end{align} and $h_{\mathbf{k}}$ is the normalization term such that the inner product of each basis function with itself is $1$.
\end{definition}

\begin{lemma} \label{lemma:fourier_decomposition}
    Following~\cite{mathew_metrics_2011}, the set of Fourier basis functions (\ref{eq:fourier_basis}) forms a set of orthonormal basis functions:
    \begin{gather}
        \langle f_{\mathbf{k}}(x), f_{\mathbf{k}}(x) \rangle = 1, \quad \forall \mathbf{k} \in \mathcal{K} \\
        \langle f_{\mathbf{k}_1}(x), f_{\mathbf{k}_2}(x) \rangle = 0, \quad \forall \mathbf{k}_1 \neq \mathbf{k}_2 \in \mathcal{K}.
    \end{gather} Furthermore, any function $g(x)$ over the same domain as the basis functions can be represented as:
    \begin{align}
        g(x) = \lim_{\#\mathcal{K}\rightarrow\infty} \sum_{\mathbf{k}\in\mathcal{K}} \langle f_{\mathbf{k}}(x), g(x) \rangle \cdot f_{\mathbf{k}}(x), 
    \end{align} where $\#\mathcal{K}$ is the number of basis functions.
\end{lemma}

\begin{definition}[Fourier ergodic metric]
    Given an $n$-dimensional spatial distribution $p(x)$ and a dynamical system $s(t)$ over a finite time horizon $[0, T]$, the Fourier ergodic metric, denoted as $\mathcal{E}$, is defined as:
    \begin{align}
        \mathcal{E}_f(s(t), p(x)) & =  \sum_{\mathbf{k}\in\mathcal{K}} \Lambda_{\mathbf{k}} \left( p_{\mathbf{k}} - c_{\mathbf{k}} \right)^2, \label{eq:fourier_metric} 
    \end{align} where the sequences of $\{p_{\mathbf{k}}\}_{\mathcal{K}}$ and $\{c_{\mathbf{k}}\}_{\mathcal{K}}$ are the sequences of Fourier decomposition coefficients of the target distribution and trajectory empirical distribution, respectively:
    \begin{align}
        p_{\mathbf{k}} = \langle p(x), f_{\mathbf{k}} (x) \rangle, \quad c_{\mathbf{k}} = \langle c_s(x), f_{\mathbf{k}} (x) \rangle.
    \end{align} The sequence of $\{\Lambda_{\mathbf{k}}\}$ is a convergent real sequence:
    \begin{align}
        \Lambda_{\mathbf{k}} = (1 + \Vert\mathbf{k}\Vert)^{-\frac{n+1}{2}} .\label{eq:fourier_lambda}
    \end{align} 
\end{definition}

\begin{lemma} \label{lemma:fourier_metric_bound}
    The Fourier ergodic metric asymptotically bounds the exact ergodic metric, as there exists two bound constants $\alpha_1, \alpha_2>0$ such that the following inequality holds with the time horizon and the number of Fourier basis functions approaching infinity:
    \begin{align}
        \alpha_1 \cdot \mathcal{E}_f(s(t), p(x)) \leq \mathcal{E}(s(t), p(x)) \leq \alpha_2 \cdot \mathcal{E}_f(s(t), p(x)). \nonumber
    \end{align}
\end{lemma}
\begin{proof}
    See Appendix A in~\cite{mathew_metrics_2011}.
\end{proof}

\begin{lemma}
    Based on Lemma~\ref{lemma:fourier_metric_bound}, the Fourier ergodic metric is asymptotically consistent with the exact ergodic metric:
    \begin{gather}
        s(t)^* = \argmin_{s(t)} \left[ \lim_{\#\mathcal{K}\rightarrow\infty} \lim_{T\rightarrow\infty} \mathcal{E}_f(s(t), p(x)) \right] \nonumber \\
        \text{ i.f.f. } s(t)^* = \argmin_{s(t)} \left[ \lim_{T\rightarrow\infty} \mathcal{E}(s(t), p(x)) \right] . \nonumber
    \end{gather} where $\#\mathcal{K}$ is the number of basis functions.
\end{lemma}

In practice, by choosing a finite number of Fourier basis functions, we can approximate the ergodicity on a system using the Fourier ergodic metric (\ref{eq:fourier_metric}) with a finite time horizon. The number of the Fourier basis functions has a significant influence on the behavior of the resulting approximated ergodic system---more Fourier basis functions will lead to better approximation but also require more computation. Past studies have revealed that the sufficient number of the basis functions for practical applications grows \emph{exponentially} with the search space dimension~\cite{sun_scale-invariant_2022}\cite{shetty_ergodic_2022}, creating a significant challenge to apply the Fourier ergodic metric in higher-dimensional spaces. In principle, the Fourier basis functions can also be defined in non-Euclidean spaces such as Lie groups. However, deriving the Fourier basis function in these spaces is non-trivial, limiting the generalization of the Fourier ergodic metric. 

In the next section we introduce our kernel ergodic metric, which is also asymptotically consistent with the exact ergodic metric but has better computational efficiency compared to the Fourier ergodic metric.


\section{Kernel Ergodic Metric} \label{sec:ergodic_metric}

\subsection{Necessary consistency condition for exact ergodic metric}

The derivation of the kernel ergodic metric is based on the following necessary condition for a metric to be consistent with the exact ergodic metric (\ref{eq:exact_metric}). 

\begin{theorem} \label{theorem:exact_metric_equiv}
   With the time horizon $T\rightarrow\infty$, a dynamic system $s(t)$ is globally optimal under the exact ergodic metric (\ref{eq:exact_metric}) with respect to the spatial distribution $p(x)$, if and only if its trajectory empirical distribution $c_s(x)$ equals to $p(x)$.
\end{theorem}
\begin{proof}
    Following Lemma~\ref{lemma:fourier_decomposition}, both the trajectory empirical distribution $c_s(x)$ and target spatial distribution $p(x)$ can be decomposed through the Fourier basis functions (\ref{eq:fourier_basis}):
    \begin{align}
        p(x) &= \lim_{\#\mathcal{K}\rightarrow\infty} \sum_{\mathbf{k}\in\mathcal{K}} p_{\mathbf{k}} \cdot f_{\mathbf{k}}(x), \quad p_{\mathbf{k}} = \langle p(x), f_{\mathbf{k}}(x) \rangle, \nonumber \\
        c_s(x) &= \lim_{\#\mathcal{K}\rightarrow\infty} \sum_{\mathbf{k}\in\mathcal{K}} c_{\mathbf{k}} \cdot f_{\mathbf{k}}(x), \quad c_{\mathbf{k}} = \langle c_s(x), f_{\mathbf{k}}(x) \rangle. \label{eq:equity_fourier_decomp}
    \end{align} From (A.14) in~\cite{mathew_metrics_2011}, the exact ergodic metric (\ref{eq:exact_metric}) can be represented as:
    \begin{align}
        \mathcal{E}(s(t), p(x)) = \lim_{\#\mathcal{K}\rightarrow\infty} \sum_{\mathbf{k}\in\mathcal{K}} a_{\mathbf{k}} \left( p_{\mathbf{k}} - c_{\mathbf{k}} \right)^2, \label{eq:exact_metric_equiv}
    \end{align} where $\{a_{\mathbf{k}}\}$ is a positive sequence defined in (A.24) of~\cite{mathew_metrics_2011}, and $\mathcal{K}$ is the number of basis functions. Based on (\ref{eq:exact_metric_equiv}), for any $s(t)$ that is globally optimal under (\ref{eq:exact_metric}), we have $p_{\mathbf{k}}{=}c_{\mathbf{k}}, \forall \mathbf{k}\in\mathcal{K}$. Therefore, we have $c_s(x) {=} p(x)$ based on (\ref{eq:equity_fourier_decomp}). 
    
    Similarly, if $p(x) {=} c_s(x)$, then we have $p_{\mathbf{k}}{=}c_{\mathbf{k}}, \forall \mathbf{k}\in\mathcal{K}$. Therefore $s(t)$ is globally optimal as $\mathcal{E}(s(t), p(x))=0$.
\end{proof}

\begin{theorem}[Necessary consistency condition] \label{theorem:metric_equiv}
    Any function $\mathcal{D}(c_s(x), p(x))$ that is globally minimized if and only if $c_s(x){=}p(x)$ is consistent with the exact ergodic metric (\ref{eq:exact_metric}). 
\end{theorem}

For a function $d(v_1, v_2)$ in a finite-dimensional vector space, one such function that satisfies the condition of being globally optimal if and only if $v_1{=}v_2$ is the commonly used quadratic formula (squared $L^2$ distance):
\begin{align}
    d(v_1, v_2) = (v_1-v_2)^\top (v_1-v_2). \label{label:vector_quad}
\end{align} We can generalize the vector space quadratic formula (\ref{label:vector_quad}) to the infinite-dimensional function space based on inner product between functions (\ref{eq:inner_product_integral}):
\begin{align}
    & \mathcal{L}(c_s(x), p(x)) \nonumber \\
    & = \langle c_s(x){-}p(x), c_s(x){-}p(x) \rangle \label{eq:func_l2_distance} \\
        & = \langle c_s(x), c_s(x) \rangle  - 2 \langle c_s(x), p(x) \rangle + \langle p(x), p(x) \rangle. \label{eq:raw_l2_distance}
\end{align} 

\begin{lemma} \label{lemma:quad_consistency}
    $\mathcal{L}(c_s(x), p(x))$ is consistent with the exact ergodic metric (\ref{eq:exact_metric}).
\end{lemma}
\begin{proof}
    Based on the positive-definite property of the inner product, $\mathcal{L}(c_s(x), p(x))$ reaches the global minima $0$ if and only if $c_s(x){=}p(x), \forall x$. Thus, based on Theorem~\ref{theorem:exact_metric_equiv}, it is consistent with the exact ergodic metric.
\end{proof}

\begin{remark}
    Although the vector space quadratic formula (\ref{label:vector_quad}) is equivalent to the squared $L^2$ distance, the function space generalization (\ref{eq:func_l2_distance}) is not necessarily equivalent to a $L^2$ distance metric between functions, since the trajectory empirical distribution $c_s(x)$ might not be in the $L^p$ space. For example, with a stationary trajectory $s(t){=}s_0,\forall t{\in}[0,T]$, $c_s(x)$ becomes a Dirac delta function $\delta(x{-}s_0)$, which is not in the $L^p$ space. 

    However, both the function space formula $\mathcal{L}(c_s(x), p(x))$ in (\ref{eq:func_l2_distance}) and Lemma~\ref{lemma:quad_consistency} only rely on the inner product between functions, which does not require either $c_s(x)$ and $p(x)$ to be in the $L^p$ space. This is common among applications of the Dirac delta functions, such as in the analysis of the position operator in quantum mechanics~\cite{cohen-tannoudji_quantum_1977}.
    
    As an example, we can show that Lemma~\ref{lemma:quad_consistency} still holds with the above case of a stationary trajectory by applying Lemma~\ref{lemma:delta_inner_product_itself} to (\ref{eq:raw_l2_distance}):
    \begin{align}
        \mathcal{L}(c_s(x), p(x)) = \delta(0) + \langle p(x),p(x)\rangle - 2 p(s_0),
    \end{align} which evaluates to the global minima of $0$ if and only if $p(x){=}\delta(x{-}s_0)$. 
\end{remark}

\begin{remark}
    Note that a regularity condition that ensures the trajectory empirical distribution $c_s(x)$ will be in the $L^p$ space is that the image of the trajectory $s(t)$, which is the support of $c_s(x)$, is compact and has a positive measure as the time horizon $T$ approaches infinity---in this case $c_s(x)$ has a finite upper bound. With a finite time horizon, this condition reduces to the requirement that the trajectory has a finite number of intersections with itself. 
\end{remark}

\subsection{Derivation of kernel ergodic metric}

We start with simplifying (\ref{eq:raw_l2_distance}) using Lemma~\ref{lemma:empirical_inner_product} and Lemma~\ref{lemma:empirical_norm}:
\begin{align}
    & \langle c_s(x), c_s(x) \rangle  - 2 \langle c_s(x), p(x) \rangle \nonumber \\
    & = - \frac{2}{T} \int_0^T p(s(t)) dt + \frac{1}{T^2} \int_{0}^{T} \int_{0}^{T} \delta(s(t_1) {-} s(t_2)) dt_1 dt_2 \nonumber \\
    & = - \frac{2}{T} \int_0^T p(s(t)) dt + \frac{1}{T^2} \int_{0}^{T} \int_{0}^{T} \phi(s(t_1), s(t_2)) dt_1 dt_2, \label{eq:simplified_l2} 
\end{align} where $\phi(\cdot,\cdot)$ is a Dirac delta \emph{kernel} function defined similarly to the Dirac delta function, as the limit of a sequence of \emph{nascent delta kernel functions} $\phi_{\theta}(\cdot,\cdot)$:
\begin{align}
    \phi(s_1, s_2) & = \lim_{\theta\rightarrow 0^+} \phi_{\theta}(s_1, s_2), \text{ } \phi_{\theta}(s_1, s_2) = \delta_{\theta}(s_1 {-} s_2) \label{eq:nascent_kernel}.
\end{align} We now formally define the \emph{kernel ergodic metric} based on (\ref{eq:simplified_l2}) and the nascent delta kernel function.

\begin{definition}[Kernel ergodic metric]
    The kernel ergodic metric $\mathcal{E}_{\theta}(s(t), p(x))$ is defined as:
    \begin{align}
        \mathcal{E}_{\theta}(s(t), p(x)) & = \frac{1}{T^2} \int_{0}^{T} \int_{0}^{T} \phi_{\theta}(s(t_1), s(t_2)) dt_1 dt_2 \nonumber \\
        & \quad - \frac{2}{T} \int_0^T p(s(t)) dt + \int p(x)^2 dx.  \label{eq:proposed_metric}
    \end{align}
\end{definition}

\begin{theorem}
    The metric (\ref{eq:proposed_metric}) is asymptotically consistent with the exact ergodic metric (\ref{eq:exact_metric}):
    \begin{gather}
        s(t)^* {=} \argmin_{s(t)} \left[ \lim_{\theta\rightarrow 0^+} \lim_{T\rightarrow\infty} \mathcal{E}_{\theta}(s(t), p(x)) \right] \nonumber \\
        \text{i.f.f. } s(t)^* {=} \argmin_{s(t)} \left[ \lim_{T\rightarrow\infty} \mathcal{E}(s(t), p(x)) \right] , \nonumber \\
        \lim_{\theta\rightarrow 0^+} \lim_{T\rightarrow\infty} \mathcal{E}_{\theta}(s(t)^*, p(x)) = \lim_{T\rightarrow\infty} \mathcal{E}(s(t)^*, p(x)) = 0. \nonumber
    \end{gather}
\end{theorem}
\begin{proof}
    Based on the delta kernel function definition (\ref{eq:nascent_kernel}) and Lemma~\ref{lemma:empirical_norm}, the following holds:
    \begin{align}
        \langle c_s(x){,} c_s(x) \rangle {=} {\lim_{\theta\rightarrow 0^+}} {\lim_{T\rightarrow\infty}} \frac{1}{T^2} {\int_{0}^{T}} {\int_{0}^{T}} \phi_{\theta}(s(t_1){,} s(t_2)) dt_1 dt_2. \nonumber
    \end{align} Therefore, the kernel ergodic metric (\ref{eq:proposed_metric}) asymptotically converges to the function space quadratic formula (\ref{eq:raw_l2_distance}). Based on Lemma~\ref{lemma:quad_consistency}, (\ref{eq:raw_l2_distance}) is consistent with the exact ergodic metric (\ref{eq:exact_metric}), thus the kernel ergodic metric is asymptotically consistent with the exact ergodic metric.
\end{proof}

There are several choices for the nascent delta kernel function (\ref{eq:nascent_kernel}) as discussed in~\cite{strichartz_guide_1994}\cite{cohen-tannoudji_quantum_1977}. For the rest of the paper, we choose the kernel function to be an isotropic Gaussian probability density function since it is differentiable and is identical to the commonly used squared exponential kernel in machine learning literature~\cite{rasmussen_gaussian_2006}:
\begin{align}
    \phi_{\theta}(s_1, s_2) = \mathcal{N}(s_1 \vert s_2, \mathbf{Id}_{\theta}), \label{eq:gaussian_nascent}
\end{align} where the covariance matrix $\mathbf{Id}_{\theta}$ is a diagonal matrix with diagonal entries specified by the parameter $\theta$. In this paper, we choose a single scalar $\theta$ for all the diagonal values for ergodic exploration in the Euclidean space. However, $\theta$ can also be a vector to specify the diagonal value for each dimension separately. Even though we choose the Gaussian kernel for experiment and illustration in this paper, all the derivations in the rest of the paper hold for any kernel function defined based on the nascent delta kernel function (\ref{eq:delta_inner_product}) that is symmetric and stationary.

\begin{figure*}[htbp]
    \centering
    \includegraphics[width=\textwidth]{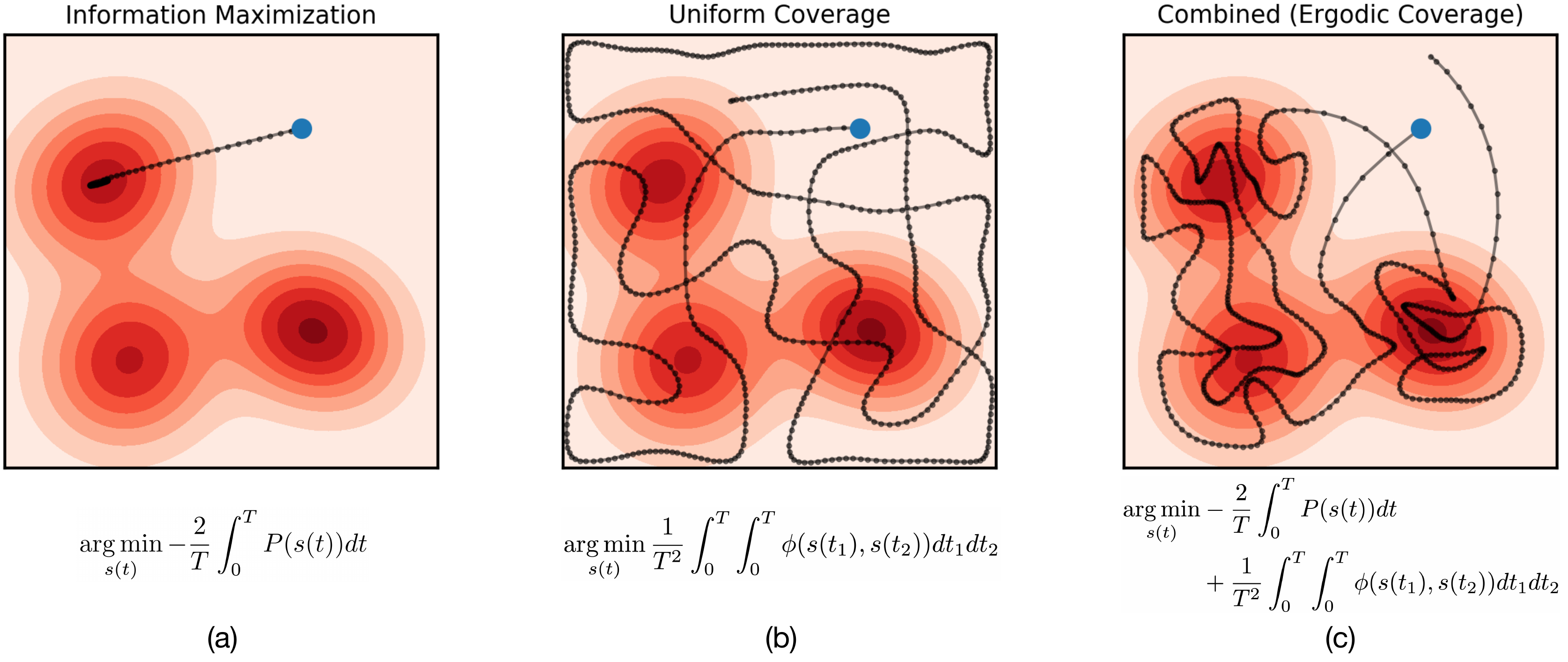}
    \caption{Trajectories when optimizing the individual and combined elements of the kernel ergodic metric (\ref{eq:proposed_metric}). (a) When only optimizing the maximum likelihood estimation term, the system is driven to a (local) maximum of the probability density; (b) When only optimizing the inner product of the trajectory empirical distribution with itself, the system uniformly covers the search space; (c) The kernel metric is the combination of the two elements, optimizing which drives the system to optimally cover the \emph{probability distribution}.}
    \label{fig:metric_elements_comparison}
\end{figure*}

\begin{figure*}[htbp]
    \centering
    \includegraphics[width=0.99\textwidth]{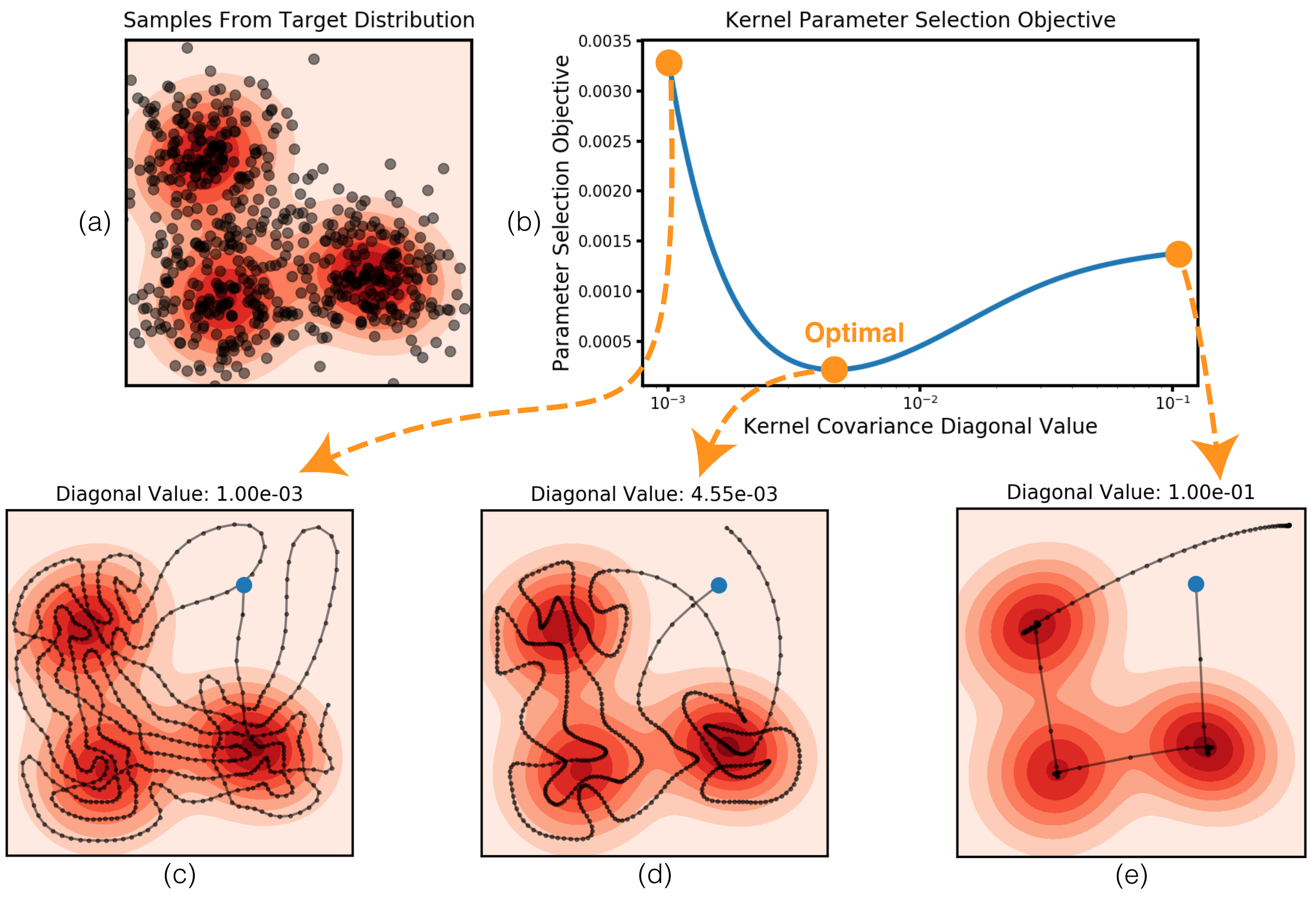}
    \caption{(a) Samples from a target distribution. (b) The kernel parameter selection objective function (\ref{eq:kernel_parameter_obj}) with the given samples. In this case, the kernel parameter is the value of the diagonal entry in the covariance. (c) A sub-optimal kernel parameter leads to an ``over-uniform'' coverage behavior. (d) The optimal kernel parameter generates an ergodic trajectory that allocates the time it spends in each region to be proportional to the integrated probability density of the region. (e) Another sub-optimal kernel parameter leads to an ``over-concentrated'' coverage behavior.}
    \label{fig:kernel_parameter_obj}
\end{figure*}

\subsection{Intuition behind kernel ergodic metric}

The kernel ergodic metric (\ref{eq:proposed_metric}) is based on and asymptotically converges to (\ref{eq:raw_l2_distance}), which involves the summation of $-2\langle c_s(x), p(x) \rangle$ and $\langle c_s(x), c_s(x) \rangle$. From Lemma~\ref{lemma:empirical_inner_product}, it is clear that minimizing $-2\langle c_s(x), p(x) \rangle$ is equivalent to maximum likelihood estimation (information maximization), thus drives the system to the state of maximum density. On the other hand, as shown below, minimizing $\langle c_s(x), c_s(x) \rangle$---the inner product of the trajectory empirical distribution with itself---drives the system to cover the search space uniformly.

\begin{lemma} \label{lemma:minimal_norm_uniformality}
A trajectory $s(t)$ that minimizes $\langle c_s(x), c_s(x) \rangle$ uniformly covers the search space $\mathcal{S}$.
\end{lemma}
\begin{proof}
See appendix.
\end{proof}

Based on the above result, we can see that the kernel ergodic metric combines---thus an ergodic system exhibits---two kinds of behavior that are both crucial for an exploration task: information maximization and uniform coverage. Figure~\ref{fig:metric_elements_comparison} showcases the trajectories from optimizing the kernel ergodic metric and each of the two terms separately.

\subsection{Automatic selection of optimal kernel parameter}
\label{subsec:kernel_selection}

In practice, the parameter $\theta$ of the Gaussian nascent delta kernel function (\ref{eq:gaussian_nascent}) plays an important role. In this section, we discuss the principle of choosing the optimal kernel parameter. Our principle is based on the observation that, i.i.d. samples from the target spatial distribution can be viewed as the trajectory of an ergodic system. We formally introduce this observation in the lemma below.

\begin{lemma}\label{lemma:large_number}
    Denote $\bar{s}=\{s_t\}$ as a discrete time series with $N$ time steps in total, where each state $s_t\sim p(x)$ is an i.i.d. sample from the target spatial distribution $p(x)$, then the system is ergodic:
    \begin{align}
        \bar{s} = \argmin_{s} \left[ \lim_{\theta\rightarrow 0^+} \lim_{N\rightarrow\infty} \mathcal{E}_\theta(s, p(x)) \right] .\label{eq:empirical_convergence}
    \end{align}
\end{lemma} 
\begin{proof}
    Based on the strong law of large numbers~\cite{vaart_asymptotic_1998}, the empirical distribution of $\bar{s}$ converges to the spatial distribution with the number of samples approaching infinity:
    \begin{align}
        \lim_{\theta\rightarrow 0^+}\lim_{N\rightarrow\infty} c_{\bar{s}}(x) = \lim_{\theta\rightarrow 0^+}\lim_{N\rightarrow\infty} \left[ \frac{1}{N} \sum_{t=1}^{N} \delta_{\theta}(x - s_t) \right] = p(x). \nonumber
    \end{align} Based on Theorem~\ref{theorem:exact_metric_equiv}, $\bar{s}$ is an ergodic system.
\end{proof}

Based on the observation in Lemma~\ref{lemma:large_number}, given a finite set of samples $\{s_i\}$ from the target spatial distribution $p(x)$, we can choose optimal nascent kernel parameter $\theta$ that minimizes the derivative of the samples $\{s_i\}$ under the kernel ergodic metric (\ref{eq:proposed_metric}), with the continuous-time integral replaced by discrete Monte-Carlo integration. In other words, we can define an optimization objective function to automatically select the optimal kernel parameters given the set of samples $\{s_i\}$.

\begin{definition}[Kernel parameter selection objective]
    Given a target spatial distribution $p(x)$, a vector set of i.i.d. samples from the distribution $\bar{s}=[s_i]$, and a parametric nascent delta kernel function $\phi_\theta(\cdot,\cdot)$, the optimal kernel parameter is selected by minimizing the following objective function:
    \begin{align}
        J(\theta) {=} \left\Vert \frac{d}{d\bar{s}} \left( {-}\frac{1}{N} \sum_{i=1}^{N} P(s_i) {+} \frac{1}{N^2} \sum_{i=1}^{N} \sum_{j=1}^{N} \phi_\theta(s_{i}, s_{j}) \right)  \right\Vert^2. \label{eq:kernel_parameter_obj}
    \end{align}
\end{definition}

\begin{remark}
    Even though we specify the nascent delta kernel to be a Gaussian kernel in this paper, the kernel parameter select objective function (\ref{eq:kernel_parameter_obj}) applies to any smooth parametric nascent delta kernel functions.
\end{remark}

In Figure~\ref{fig:kernel_parameter_obj}, an example objective function for kernel parameter selection is shown, as well as how different kernel parameters could influence the resulting ergodic trajectory. From Figure~\ref{fig:kernel_parameter_obj}, we can also see that the kernel parameter is an adjustable parameter for a practitioner to generate coverage trajectories that balance behaviors between uniform coverage and seeking information maximization. Thus, a kernel parameter could be sub-optimal under the parameter selection objective yet still generate valuable trajectories for practitioners depending on the specific requirements of a task.


\section{Optimal Control With Kernel Ergodic Metric} \label{sec:ergodic_control}

In this section, we will introduce the method to optimize the kernel ergodic metric when the trajectory is governed by a continuous-time dynamic system. Our optimal control formula is based on the continuous-time iterative linear quadratic regulator (iLQR) framework~\cite{hauser_projection_2002}, which is also used as the optimal control framework for the Fourier ergodic metric in~\cite{miller_trajectory_2013-1}\cite{miller_ergodic_2016}. We will first introduce the preliminaries of the iLQR algorithm.

\subsection{Preliminaries for iterative linear quadratic regulator}

The continuous-time iterative linear quadratic regulator (iLQR) method finds the local optimum of the following (nonlinear) optimal control problem:
\begin{align}
    u^* & = \argmin_{u} J(u) \\
        & = \argmin_{u} \int_0^T l(s(t), u(t)) dt \label{eq:oc_standard_obj} \\ 
    \text{s.t. } s(t) & = s_0 + \int_0^t f(s(\tau), u(\tau)) d\tau,
\end{align} where $l(s(t), u(t))$ is the runtime cost function. Both the cost function and the dynamics $f(s(t),u(t))$ can be nonlinear. 

In each iteration of the continuous-time iLQR framework, we find an optimal descent direction $v(t)$ of the current control $u(t)$ by solving the following optimal control problem:
\begin{align}
    v^* = \argmin_{v} DJ(u) {\cdot} v + \int_0^T \Vert z(t) \Vert^2_Q + \Vert v(t)\Vert^2_R dt, \label{eq:ilqr_subprobelm}
\end{align} where $Q$ and $R$ are user-specified regulation matrices, $z(t)$ is the corresponding perturbation on the system state $s(t)$ by applying the control perturbation $v(t)$, and $DJ(u) {\cdot} v$ is the Gateaux derivative of the objective function in the direction of $v(t)$ defined as: 
\begin{align}
    DJ(u) {\cdot} v = \lim_{\epsilon\rightarrow 0} \frac{d}{d\epsilon} J(u + \epsilon {\cdot} v).
\end{align} The following lemma show that this subproblem (\ref{eq:ilqr_subprobelm}) is a linear quadratic regulator (LQR) problem.

\begin{lemma} \label{lemma:lqr_subproblem}
    The Gateaux derivative of the cost function $J(u)$ defined in (\ref{eq:oc_standard_obj}) can be written as:
    \begin{align}
        DJ(u) {\cdot} v = \int_0^T a(t)^\top z(t) + b(t)^\top v(t) dt, 
    \end{align} where 
    \begin{align}
        a(t) = \frac{d}{d s(t)} l(s(t), u(t)), \text{ } b(t) = \frac{d}{d u(t)} l(s(t), u(t)). \label{eq:linearized_cost}
    \end{align} Furthermore, the perturbation $z(t)$ on the state trajectory $s(t)$ has a linear dynamics:
    \begin{align}
        z(t) = z_0 + \int_0^T A(\tau) z(\tau) + B(\tau) v(\tau) d\tau, \text{ } z_0 = 0, \label{eq:linearized_dynamics}
    \end{align} where
    \begin{align}
        A(\tau) & = \frac{d}{d s(\tau)} f(s(\tau), u(\tau)), \text{ } B(\tau) = \frac{d}{d u(\tau)} f(s(\tau), u(\tau)). \nonumber
    \end{align}
\end{lemma}
\begin{proof}
    See \cite{hauser_projection_2002} and \cite{miller_trajectory_2013-1}.
\end{proof} 

Since the subproblem in (\ref{eq:ilqr_subprobelm}) is a standard continuous-time LQR problem, we can find the optimal descent direction by solving the continuous-time Riccati equation. After solving the LQR subproblem (\ref{eq:ilqr_subprobelm}), we can update the control $u(t)$ along with the optimal descent direction, with a step size that can be found using the Armijo backtracking line search~\cite{nocedal_numerical_2006}. 

\subsection{Derive iLQR for kernel ergodic metric}

\begin{definition}[Kernel ergodic control]
    Given a target distribution $p(x)$ and system dynamics $\dot{s}(t) = f(s(t), u(t))$, the kernel ergodic control problem is defined as follow:
    \begin{align}
        u^* & = \argmin_{u} J(u) \\
        J(u(t)) & = \mathcal{E}_{\theta}(s(t),p(x)) + \int_0^T l(s(t), u(t)) dt \label{eq:ergodic_control_obj} \\
        \text{s.t. } & s(t) = s_0 + \int_{0}^{t} f(s(\tau), u(\tau)) d\tau,
    \end{align} where $\mathcal{E}_{\theta}(s(t),p(x))$ is the kernel ergodic metric (\ref{eq:proposed_metric}) and $l(s(t),u(t))$ is the additional run-time cost, such as the regulation cost on the control. 
\end{definition}

The only difference between the kernel ergodic control problem and the optimal control objective in (\ref{eq:oc_standard_obj}) is that the kernel ergodic metric (\ref{eq:proposed_metric}) has a double time integral instead of a single time integral. Therefore, we need to derive the Gateaux derivative of the kernel ergodic metric in order to apply iLQR to the kernel ergodic control problem.

\begin{lemma} \label{lemma:gateaux_kernel}
    The Gateaux derivative of the kernel ergodic metric is:
    \begin{gather}
        D \mathcal{E}_{\theta}(s(t),p(x)) \cdot z(t) = \int_0^T a_\theta(t) z(t) dt, \\
        a_\theta(t) = -\frac{2}{T} \frac{d}{ds(t)} p(s(t)) + \frac{2}{T^2} \int_{0}^{T} \frac{d}{d s(t)}\phi(s(t), s(\tau)) d\tau. \nonumber
    \end{gather}
\end{lemma}
\begin{proof}
    See appendix.
\end{proof} 

With Lemma~\ref{lemma:gateaux_kernel}, we specify the LQR subproblem to be solved in each iteration for kernel ergodic control.

\begin{definition}[LQR subproblem]
    The iLQR algorithm for kernel ergodic control (\ref{eq:ergodic_control_obj}) iteratively solves the following LQR problem through the continuous-time Riccati equation to compute the optimal descent direction to update the control:
    \begin{align}
        v^* & = \argmin_{v} \int_0^T \Vert z(t) \Vert^2_Q + \Vert v(t)\Vert^2_R \nonumber \\
        & \quad\quad\quad\quad\quad\quad + (a_\theta(t){+}a(t))^\top z(t) + b(t)^\top v(t) dt \label{eq:kernel_lqr_subproblem} \\
        \text{s.t. } & z(t) = z_0 + \int_0^T A(\tau) z(\tau) + B(\tau) v(\tau) d\tau, \text{ } z_0 {=} 0 .
    \end{align}
\end{definition} The pseudocode of the iLQR algorithm for kernel ergodic control is described in Algorithm~\ref{algo:traj_opt}.

\begin{algorithm} 
    \caption{Kernel-ergodic trajectory optimization}
    \label{algo:traj_opt}
    \begin{algorithmic}[1] 
        \Procedure{TrajOpt}{$s_0$, $\bar{u}(t)$}
        \State $k \gets 0$ \Comment{$k$ is the iteration index.}
        \State $u_k(t) \gets \bar{u}(t)$
        \While{termination criterion not met}
            \State Simulate $s_k(t)$ given $s_0$ and $u_k(t)$
            \State Compute descent direction $v_k(t)$ by solving  Eq(\ref{eq:kernel_lqr_subproblem})
            \State Find step size $\eta$ \Comment{E.g., apply line search}
            \State $u_{k+1}(t) \gets u_k(t) + \eta\cdot v_k(t)$
            \State $k \gets k+1$
        \EndWhile
        \State \textbf{return} $u_k(t)$
        \EndProcedure
    \end{algorithmic}
\end{algorithm}

\subsection{Accelerating optimization} 

We further introduce two approaches to accelerate the computation in Algorithm~\ref{algo:traj_opt}. 

\subsubsection{Bootstrap} \label{subsec:bootstrap}
The bootstrap step generates an initial trajectory roughly going through the target distribution. We formulate a trajectory tracking problem with the reference trajectory as an \emph{ordered} set of samples from the target distribution. The order of the samples is determined by approximating the solution of a traveling-salesman problem (TSP) through the nearest-neighbor approach~\cite{rosenkrantz_analysis_1977}, which has a maximum quadratic complexity. 

\subsubsection{Parallelization} 

The time integral term in the descent direction formula (\ref{eq:kernel_lqr_subproblem}) and the kernel ergodic metric itself (\ref{eq:proposed_metric}) can be computed using the Riemann sum formula, which can be computed in parallel.


\section{Kernel Ergodic Control on Lie groups} \label{sec:lie_groups}

So far, the derivation of the kernel ergodic metric and the trajectory optimization method assumes the robot state evolves in an Euclidean space. One of the advantages of the kernel ergodic metric is that it can be generalized to other Riemannian manifolds, particularly Lie groups. 

\subsection{Preliminaries}

A Lie group is a smooth manifold. Thus, any element on the Lie group locally resembles a linear space. However, unlike other manifolds, elements in a Lie group also satisfy the four group axioms equipped with a composition operation: closure, identity, inverse, and associativity. In robotics, the Lie group is often used to represent non-linear geometrical spaces, such as the space of rotations or rigid body transformations, while allowing analytical and numerical techniques in the Euclidean space to be applied. In particular, we are interested in the special orthogonal group SO(3) and the special Euclidean group SE(3), which are used extensively to model 3D rotation and 3D rigid body transformation (simultaneous rotation and translation), respectively. Below, we briefly introduce the key concepts of Lie groups that allow us to generalize the kernel ergodic control framework to Lie groups. For more information on Lie groups and their application in robotics, we refer the readers to~\cite{choset_principles_2005,chirikjian_stochastic_2009,lynch_modern_2017,sola_micro_2021,boumal_introduction_2023}.

\begin{definition}[SO(3) group]
    The special orthogonal group SO(3) is a matrix manifold in which each element is a 3-by-3 matrix satisfying the following property:
    \begin{align}
        g^\top g = g g^\top = I \text{ and } \det(g) = 1, \quad \forall g\in SO(3) \subset \mathbb{R}^{3\times 3}, \nonumber
    \end{align} where $I$ is a 3-by-3 identify matrix. The composition operator for SO(3) is the standard matrix multiplication. 
\end{definition}

\begin{definition}[SE(3) group]
    The special Euclidean group SE(3) is a matrix manifold. Each element of SE(3) is a 4-by-4 matrix that, when used as a transformation between two Euclidean space points, preserves the Euclidean distance between and the handedness of the points. Each element has the following structure:
    \begin{align}
        g = \begin{bmatrix}
            R & \mathbf{t} \\
            \mathbf{0} & 1
        \end{bmatrix}, R \in SO(3), \mathbf{t}, \mathbf{0}\in\mathbb{R}^3.
    \end{align} The composition operation in SE(3) is simply the standard matrix multiplication and it has the following structure:
    \begin{align}
        & g_1 \circ g_2 = \begin{bmatrix}
            R_1 R_2 & R_1 \mathbf{t}_2 + \mathbf{t}_1 \\
            \mathbf{0} & 1 
        \end{bmatrix} \\
        g_1 & = \begin{bmatrix}
            R_1 & \mathbf{t}_1 \\
            \mathbf{0} & 1
        \end{bmatrix}, \quad g_2 = \begin{bmatrix}
            R_2 & \mathbf{t}_2 \\
            \mathbf{0} & 1 
        \end{bmatrix}. \nonumber
    \end{align}
\end{definition}

The smooth manifold property of the Lie group means at every element in SO(3) and SE(3), we can locally define a linear matrix space. We call such space the \emph{tangent space} of the group.

\begin{definition}[Tangent space]
    For an element $g$ on a manifold $\mathcal{M}$, its tangent space $\mathcal{T}_g\mathcal{M}$ is a linear space consisting of all possible tangent vectors that pass through $g$.
\end{definition}
\begin{remark}
    Each element in the tangent space $\mathcal{T}_g\mathcal{M}$ can be considered as the time derivative of a temporal trajectory on the manifold $g(t)$ that passes through the $g$ at time $t$. Given the definition of a Lie group, the time derivative of such a trajectory is a vector.
\end{remark}

\begin{definition}[Lie algebra]
    For a Lie group $\mathcal{G}$, the tangent space at its identity element $\mathcal{I}$ is the Lie algebra of this group, denoted as $\mathfrak{g} = \mathcal{T}_{\mathcal{I}}\mathcal{G}$.
\end{definition}

Despite being a linear space, the tangent space on the Lie group and Lie algebra could still have non-trivial structures. For example, the Lie algebra of the SO(3) group is the linear space of skew-symmetric matrices. However, elements in Lie algebra can be expressed as a vector on top of a set of \emph{generators}, which are the derivatives of the tangent element in each direction. This key insight allows us to represent Lie algebra elements in the standard Euclidean vector space. We can transform the elements between the Lie algebra and the standard Euclidean space through two isomorphisms---the \emph{hat} and \emph{vee} operators---defined below. 

\begin{figure*}[htbp]
    \centering
    \includegraphics[width=0.9\textwidth]{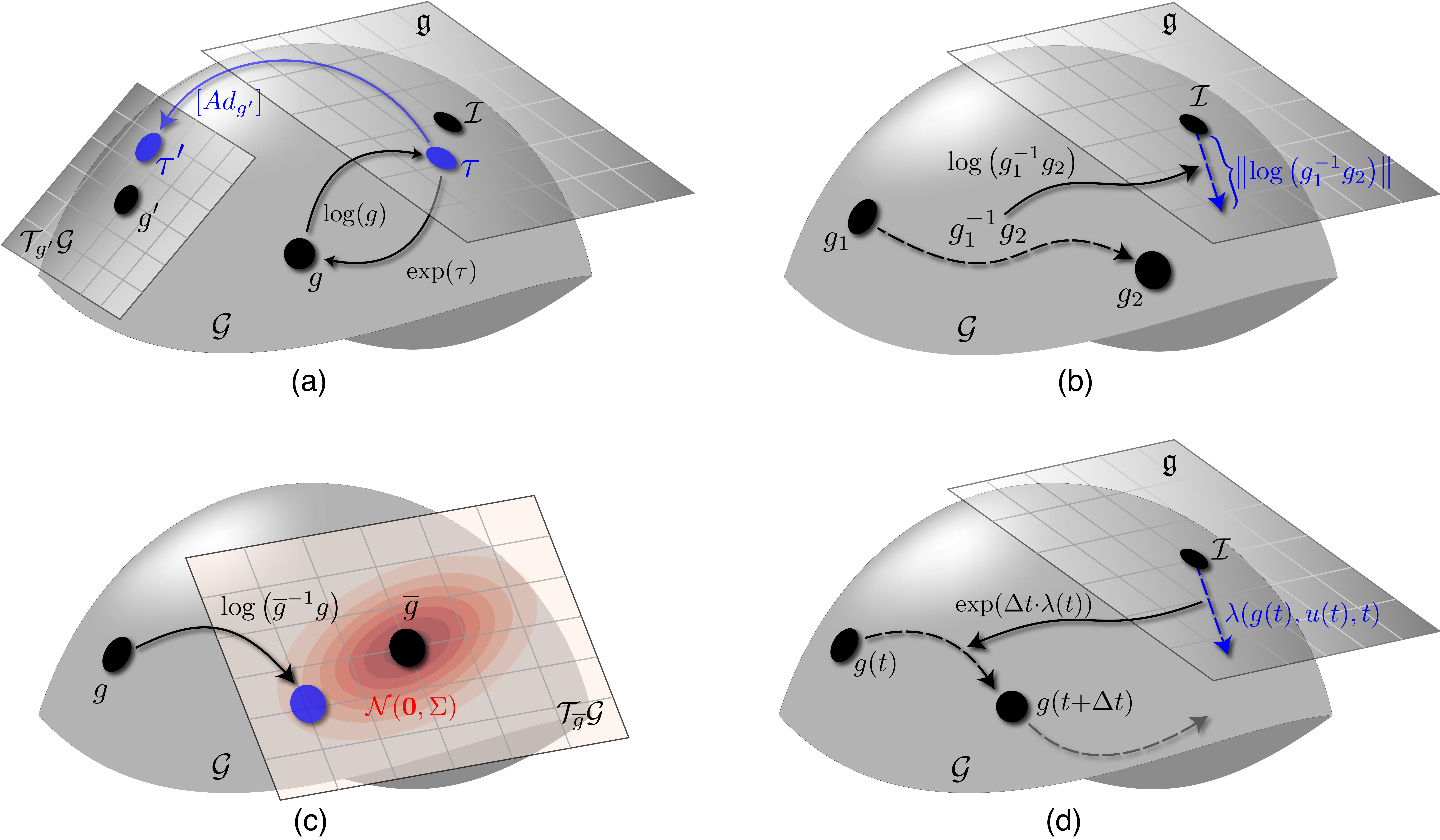}
    \caption{Illustration of key concepts in the Lie group ergodic search formula. (a) The exponential map maps a Lie algebra element $\tau\in\mathfrak{g}$ to a Lie group element $g\in\mathcal{G}$; The logarithm map is the inverse of the exponential map; The adjoint transformation maps an element from one tangent space (Lie algebra in this case) to an element in another tangent space $\mathcal{T}_{g^\prime}\mathcal{G}$. (b) The difference between two Lie group elements $g_1^{-1}g_2$ is mapped to the Lie algebra $\log(g_1^{-1}g_2)$ through the logarithm map, which allows us to use the Euclidean space formula to define the quadratic function on the Lie group. (c) The Lie group Gaussian distribution is defined in the tangent space of the mean $\bar{g}$. The probability density function is evaluated as a zero-mean Euclidean Gaussian distribution $\mathcal{N}(\mathbf{0},\Sigma)$ over the Lie group sample $g$ in the tangent space $\log(\bar{g}^{-1}g)$. (d) Dynamics is defined through the left-trivialization in the Lie algebra $\lambda:\mathcal{G}\times\mathbb{R}^n\times\mathbb{R}_0^+\mapsto\mathfrak{g}$, which is mapped back to propagate the Lie group system state through the exponential map $\exp(\Delta t{\cdot}\lambda(t))$. The dynamics is defined as continuous, but the Lie group trajectory is integrated piece-wise.}
    \label{fig:lie_group_intro}
\end{figure*}

\begin{definition}[Hat]
    The hat operator $\hat{(\cdot)}$ is an isomorphism from a $n$-dimensional Euclidean vector space to the Lie algebra with $n$ degress of freedom:
    \begin{align}
        \hat{(\cdot)} : \mathbb{R}^n \mapsto \mathfrak{g}; \quad \hat{\nu} = \sum_{i=1}^{n} \nu_i E_i \in \mathfrak{g}, \quad \nu\in\mathbb{R}^n,
    \end{align} where $E_i$ is the $i$-th generator of the Lie algebra. 
\end{definition}

\begin{definition}[Vee]
    The vee operator $^\vee{(\cdot)} : \mathfrak{g} \mapsto \mathbb{R}^n$ is the inverse mapping of the hat operator. 
\end{definition}

For the SO(3) group, the hat operator is defined as:
\begin{align}
    \hat{\omega} & = \begin{bmatrix}
        0 & -\omega_3 & \omega_2 \\
        \omega_3 & 0 & -\omega_1 \\
        -\omega_2 & \omega_1 & 0
    \end{bmatrix}, \omega \in \mathbb{R}^3.
\end{align} For the SE(3) group, the hat operator is defined as:
\begin{align}
    \hat{\tau} = \begin{bmatrix}
        \hat{\omega} & \nu \\
        \mathbf{0} & 0
    \end{bmatrix} \in \mathbb{R}^{4\times 4}, \quad \tau=\begin{bmatrix}
        \omega\\ \nu
    \end{bmatrix} \in \mathbb{R}^6, \omega,\nu\in\mathbb{R}^3.
\end{align}

\begin{definition}[Exponential map]
    The exponential map, denoted as $\text{exp}:\mathfrak{g}\mapsto\mathcal{G}$, maps an element from the Lie algebra to the Lie group. 
\end{definition}

\begin{definition}[Logarithm map]
    The logarithm map, denoted as $\log:\mathcal{G}\mapsto\mathfrak{g}$, maps an element from the Lie algebra to the Lie group. 
\end{definition}

The exponential and logarithm map for the SO(3) and SE(3) groups can be computed in practice through specific, case-by-case formulas. For example, the exponential map for the SO(3) group can be computed using the Rodrigues' rotation formula. More details regarding the formulas for exponential and logarithm map can be found in~\cite{lynch_modern_2017}.

\begin{definition}[Adjoint]
    The adjoint of a Lie group element $g$, denoted as $Ad_g:\mathfrak{g}\mapsto\mathfrak{g}$, transforms the vector in one tangent space to another. Given two tangent spaces, $\mathcal{T}_{g_1}\mathcal{G}$ and $\mathcal{T}_{g_2}\mathcal{G}$, from two elements of the Lie group $\mathcal{G}$, the adjoint enables the following transformation:
    \begin{align}
        v_1 = Ad_{g_1^{-1}g_2} (v_2).
    \end{align} 
\end{definition}

Since the adjoint is a linear transformation, it can be represented as a matrix denoted as $[Ad_g]$. The adjoint matrix for a SO(3) matrix is itself, the adjoint matrix for a SE(3) matrix is:
\begin{align}
    [Ad_g] = \begin{bmatrix}
        R & \hat{\mathbf{t}} R \\
        \mathbf{0} & R
    \end{bmatrix} \in \mathbb{R}^{6\times 6}, \quad g = \begin{bmatrix}
        R & \mathbf{t} \\
        \mathbf{0} & 1
    \end{bmatrix}.
\end{align} Visual illustrations of the exponential map, logarithm map, and adjoint are shown in Figure~\ref{fig:lie_group_intro}(a).

\subsection{Kernel on Lie groups}

The definition of a Gaussian kernel is built on top of the notion of ``distance''---a quadratic function of the ``difference''---between the two inputs. While the definition of distance in Euclidean space is trivial, its counterpart in Lie groups will have different definitions and properties. Thus, to define a kernel in a Lie group, we start with defining quadratic functions in Lie groups~\cite{fan_online_2016}.

\begin{definition}[Quadratic function]
Given two elements $g_1, g_2$ on the Lie group $\mathcal{G}$, we can define the quadratic function as:
\begin{align}
    C(g_1, g_2) = \frac{1}{2} \Vert \log(g_2^{-1} g_1) \Vert_{M}^2, \label{eq:lie_quadratic}
\end{align} where $M$ is the weight matrix and $\log$ denotes Lie group logarithm. 
\end{definition}

The visual illustration of the quadratic function on Lie groups is shown in Figure~\ref{fig:lie_group_intro}(b). Since the quadratic function is defined on top of Lie algebra, it has similar numerical properties to regular Euclidean space quadratic functions, such as symmetry.

The derivatives of the quadratic function, following the derivation in~\cite{fan_online_2016}, are as follows:
\begin{align}
    D_1 C(g_1, g_2) & = \text{d} \exp^{-1}\left(-\log(g_2^{-1}g_1)\right)^\top M \log(g_2^{-1}g_1) \\
    D_2 C(g_1, g_2) & = -[\mathit{Ad}_{g_1^{-1}g_2}]^\top D_1 C(g_1, g_2),
\end{align} where $\text{d}\exp^{-1}$ denotes the trivialized
tangent inverse of the exponential map, its specification on SO(3) and SE(3) are derived in~\cite{fan_online_2016}.

Given (\ref{eq:lie_quadratic}), we now define the squared exponential kernel on Lie groups.

\begin{definition}
    The squared exponential kernel on Lie groups is defined as:
    \begin{align}
        \Phi(g_1, g_2) = \alpha \cdot \exp\left( \frac{1}{2} \Vert \log(g_2^{-1} g_1) \Vert_{M}^2 \right).
    \end{align}
\end{definition}

\subsection{Probability distribution on Lie groups}

Probability distributions in Euclidean space need to be generalized to Lie groups case by case; thus, we primarily focus on generalizing Gaussian and Gaussian-mixture distributions to the Lie group as the target distribution. The results here also apply to other probability distributions, such as the Cauchy distribution and Laplace distribution.

Our formula follows the commonly used \emph{concentrated Gaussian} formula~\cite{yunfeng_wang_error_2006,wang_nonparametric_2008,chirikjian_gaussian_2014}, which has been widely used for probabilistic state estimation on Lie groups~\cite{chauchat_invariant_2018,mangelson_characterizing_2020, hartley_contact-aided_2020}. 

\begin{definition}[Gaussian distribution]
    Given a Lie group mean $\bar{g}\in\mathcal{G}$ and a covariance matrix $\Sigma$ whose dimension matches the degrees of freedom of the Lie group (thus the dimension of a tangent space on the group), we can define a Gaussian distribution, denoted as $\mathcal{N}_{\mathcal{G}}(\bar{g}, \Sigma)$, with the following probability density function:
    \begin{align}
        \mathcal{N}_{\mathcal{G}}(g\vert \bar{g}, \Sigma) & = \mathcal{N}(\log(\bar{g}^{-1} \circ g) \vert \mathbf{0}, \Sigma) \label{eq:concentrated_gaussian},
    \end{align} where $\mathcal{N}(\mathbf{0}, \Sigma)$ is a zero-mean Euclidean Gaussian distribution in the tangent space of the mean $\mathcal{T}_{\bar{g}}\mathcal{G}$.
\end{definition} 

Given the above definition, in order to generate a sample $g{\sim}\mathcal{N}_{\mathcal{G}}(\bar{g}, \Sigma)$ from the distribution, we first generate a perturbation from the distribution the tangent space $\epsilon{\sim}\mathcal{N}(\mathbf{0}, \Sigma)$, which will perturb the Lie group mean to generate the sample:
\begin{align}
    g = \bar{g} \circ \exp(\epsilon) & \sim \mathcal{N}_{\mathcal{G}}(\bar{g}, \Sigma). \label{eq:perturb_sample_gen} \\
    \epsilon & \sim \mathcal{N}(\mathbf{0}, \Sigma)
\end{align} Following this relation, we can verify that the Lie group Gaussian distribution and the tangent space Gaussian distribution share the same covariance matrix through the following equation:
\begin{align}
    \Sigma & = \mathbb{E}\left[ \epsilon\epsilon^\top \right] \\
     & = \mathbb{E}\left[ \log(\bar{g}^{-1}\circ g) \log(\bar{g}^{-1}\circ g)^\top \right].
\end{align}

Since the optimal control formula requires the derivative of the target probability density function with respect to the state, we now give the full expression of the probability density function and derive its derivative:
\begin{align}
    P(g) & = \mathcal{N}_{\mathcal{G}}(g\vert \bar{g}, \Sigma) \nonumber \\
    & = \mathcal{N}(\log(\bar{g}^{-1}\circ g) \vert \mathbf{0}, \Sigma) \nonumber \\
    & = \eta \cdot \exp\left( -\frac{1}{2}\log\left( \bar{g}^{-1}g\right)^\top\Sigma^{-1}\log\left( \bar{g}^{-1}g\right) \right), 
\end{align} where $\eta$ is the normalization term defined as:
\begin{align}
    \eta = \frac{1}{\sqrt{(2\pi)^n\det(\Sigma)}}. 
\end{align} The derivative of $P(g)$ is:
\begin{align}
    D P(g) = P(g) \cdot -\left( \frac{d}{dg} \log\left(\bar{g}^{-1}g\right) \right)^\top \Sigma^{-1}\log\left(\bar{g}^{-1}g\right), 
\end{align} where the derivative $\frac{d}{dg} \log\left(\bar{g}^{-1}g\right)$ can be further expanded as:
\begin{align}
    \frac{d}{dg} \log\left(\bar{g}^{-1}g\right) & = \mathbf{d}\exp\left(-\log\left(\bar{g}^{-1}g\right)\right) \cdot \frac{d}{dg}\left(\bar{g}^{-1}g \right) \\
    & = \mathbf{d}\exp\left(-\log\left(\bar{g}^{-1}g\right)\right), 
\end{align} where $\text{d}\exp$ and $\text{d}\exp^{-1}$ denote the trivialized
tangent of the exponential map and the inverse of the exponential map, the specification of two on SO(3) and SE(3) are derived in~\cite{fan_online_2016}.

\begin{remark}
    Our formula of concentrated Gaussian distribution on Lie groups perturbs the Lie group mean on the right side (\ref{eq:perturb_sample_gen}). Another formula is to perturb the mean on the left side. The Lie group derivation of the kernel ergodic metric holds for both formulas. As discussed in~\cite{mangelson_characterizing_2020}, the primary difference between the two formulas is the frame in which the perturbation is applied.
\end{remark}

\begin{remark} \label{remark:compact_gaussian}
    Although commonly used in robotics, the concentrated Gaussian distribution formula (\ref{eq:concentrated_gaussian}) has one limitation on compact Lie groups, such as SO(3), compared to the standard Euclidean space Gaussian formula. The eigenvalues of the covariance matrix need to be sufficiently small such that the probability density function diminishes to zero on a small sphere centered around the mean (hence the name ``concentrated''), in which case the global topological properties of the group (e.g., compactness) are not relevant~\cite{chirikjian_gaussian_2014}. 
\end{remark}

\begin{figure*}[htbp]
    \centering
    \includegraphics[width=0.9\textwidth]{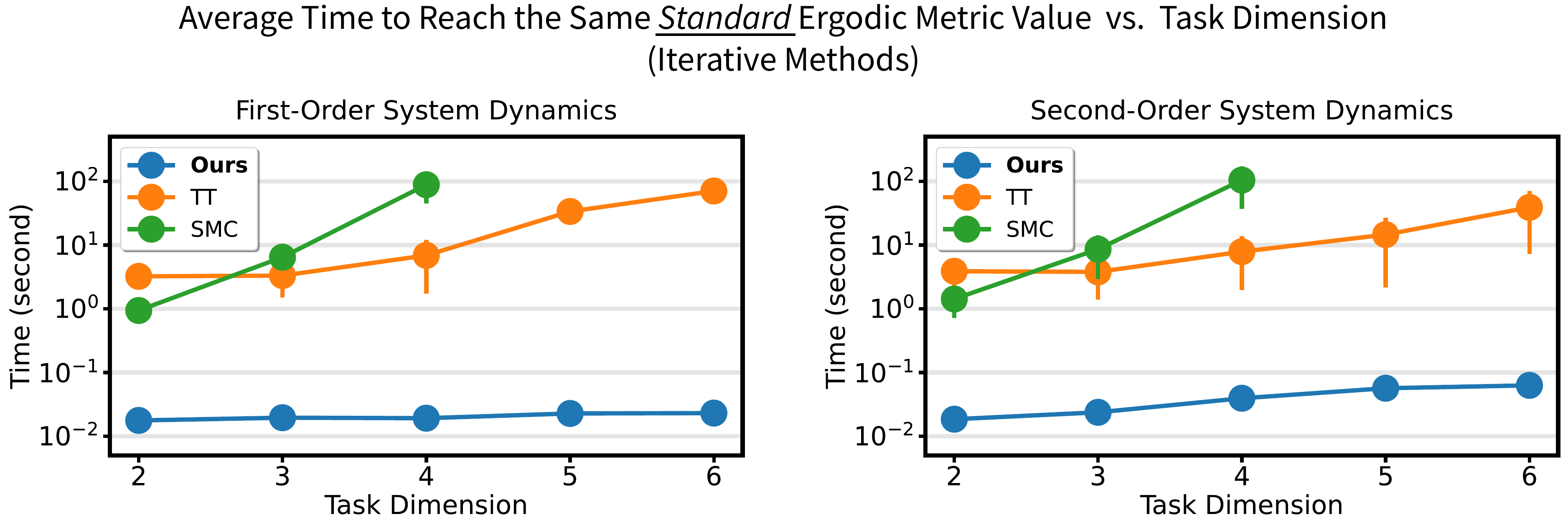}
    \caption{Comparison of the scalability of different methods. The proposed method exhibits a linear complexity across 2 to 6-dimensional spaces, while the Fourier metric-based methods, even if accelerated by tensor-train, exhibit a close-to-exponential complexity.}
    \label{fig:scalability_plot}
    \vspace{-1em}
\end{figure*}

\subsection{Dynamics on Lie groups}

Given a trajectory evolving on the Lie group $g(t):[0,T]\mapsto\mathcal{G}$, we define its dynamics through a control vector field~\cite{saccon_optimal_2013}:
\begin{align}
    \dot{g}(t) = f(g(t), u(t), t) \in \mathfrak{g}.
\end{align} In order to linearize the dynamics as required by the trajectory optimization algorithm in (\ref{eq:linearized_dynamics}), we follow the derivation in~\cite{saccon_optimal_2013} to model the dynamics through the \emph{left trivialization} of the control vector field:
\begin{align}
    \lambda(g(t), u(t), t) = g(t)^{-1} f(g(t), u(t), t) \in \mathfrak{g},
\end{align} which allows us to write down the dynamics instead as:
\begin{align}
    \dot{g}(t) = g(t) \lambda(g(t), u(t), t).
\end{align} {To propagate the Lie group state between discrete time steps $t$ and $t{+}dt$, we have~\cite{kobilarov_discrete_2011}:
\begin{align}
    g(t{+}dt) = g(t) \exp\Big(dt {\cdot} f(g(t), u(t), t)\Big). \label{eq:lie_time_integral}
\end{align} The resulting trajectory is a piece-wise linearized approximation of the continuous Lie group dynamic system~\cite{kobilarov_discrete_2011}.

Denote a perturbation on the control $u(t)$ as $v(t)$ and the resulting tangent space perturbation on the Lie group state as $z(t)\in\mathfrak{g}$, $z(t)$ exhibits a similar linear dynamics as its Euclidean space counterpart, as shown in Lemma~\ref{lemma:lqr_subproblem}:
\begin{align}
    \dot{z}(t) & = A(t) z(t) + B(t) v(t) \\
    A(t) & = D_1 \lambda(g(t), u(t), t) - [Ad_{\lambda(g(t), u(t), t)}] \\
    B(t) & = D_2 \lambda(g(t), u(t), t).
\end{align} Since the linearization of the dynamics is in the tangent space, this allows us to directly apply Algorithm~\ref{algo:traj_opt} to optimize the control, where the descent direction is computed by solving a continuous-time Riccati equation using standard approaches in the Euclidean space. We refer readers to~\cite{kobilarov_discrete_2011,saccon_optimal_2013} for more details on the dynamics and optimal control of Lie groups.


\section{Evaluation} \label{sec:evaulation}

\subsection{Overview}

We first evaluate the numerical efficiency of our algorithm compared to existing ergodic search algorithms through a simulated benchmark. We then demonstrate our algorithm, specifically the Lie group SE(3) variant, with a peg-in-hole insertion task.

\subsection{Numerical Benchmark}

\begin{figure*}[htbp]
    \includegraphics[width=0.99\textwidth]{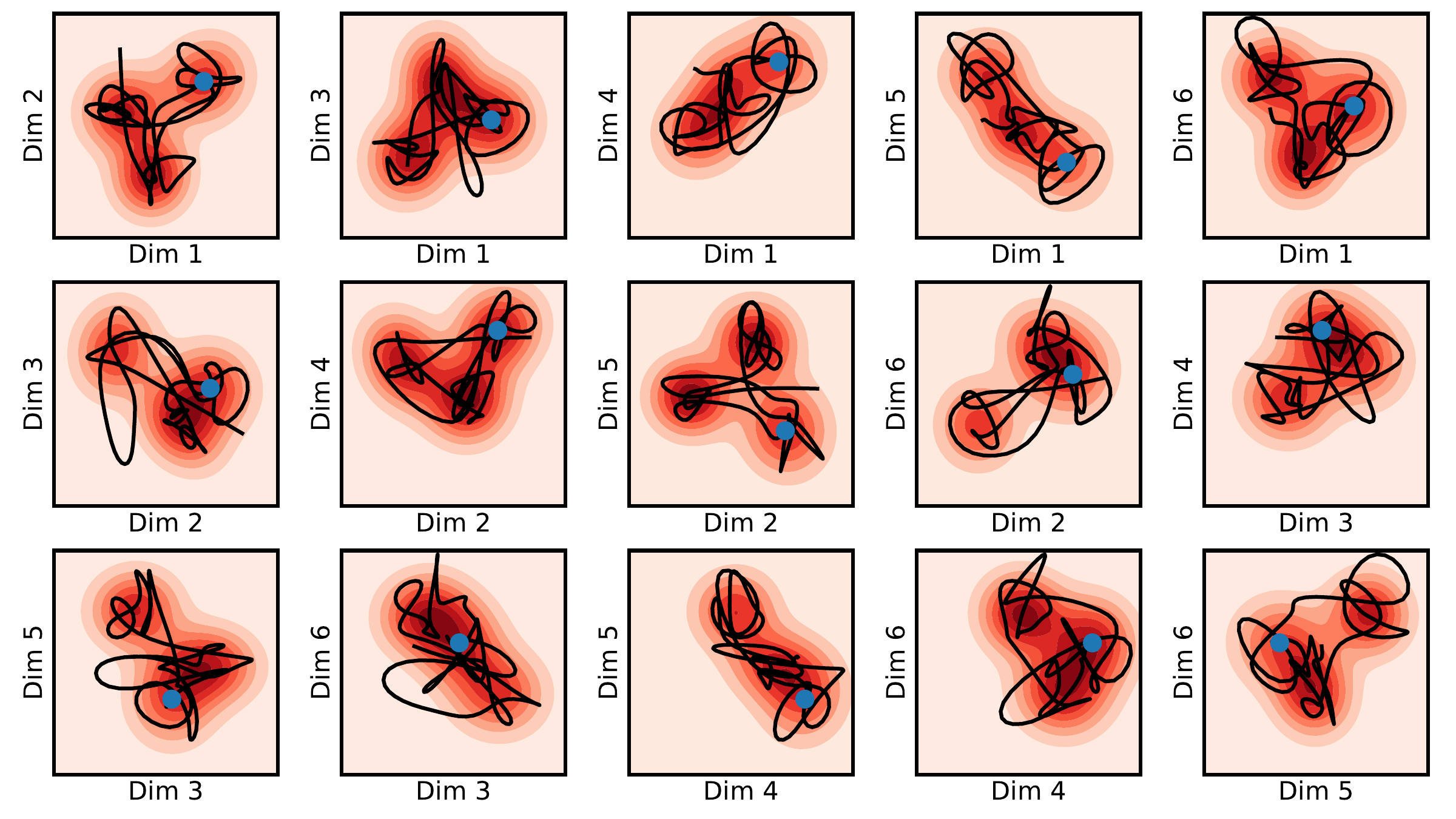}
    \caption{Example ergodic trajectory generated by the proposed algorithm in 6-dimensional space, with second-order system dynamics. The trajectory overlaps the marginalized target distribution.}
    \label{fig:sim_traj_6d}
\end{figure*}

\begin{table*}[htbp]
    \centering
    \captionsetup{justification=centering}
    \caption{Average time required for \textbf{\textcolor{blue}{iterative}} methods to reach the same ergodic metric value\\(\textbf{first-order} system dynamics).}
    \setlength{\tabcolsep}{20.0pt}
    \begin{tabular}{cccccc}
        \toprule
        \midrule
        \multirow{3}{*}{\begin{tabular}[c]{@{}c@{}}Task\\Dim.\end{tabular}} & \multirow{3}{*}{\begin{tabular}[c]{@{}c@{}}System\\Dim.\end{tabular}} & \multirow{3}{*}{\begin{tabular}[c]{@{}c@{}}Average Target\\Ergodic Metric Value\end{tabular}} & \multicolumn{3}{c}{Average Elapsed Time (second)}  \\
        \cmidrule(lr){4-6} 
        & & & \multirow{2}{*}{\begin{tabular}[c]{@{}c@{}}\textsf{Ours}\\(Iterative)\end{tabular}} & \multirow{2}{*}{\begin{tabular}[c]{@{}c@{}}\textsf{TT}\\(\textcolor{blue}{Iterative})\end{tabular}} & \multirow{2}{*}{\begin{tabular}[c]{@{}c@{}}\textsf{SMC}\\(\textcolor{blue}{Iterative})\end{tabular}} \\
         & & & & & \\
        \midrule
        \midrule
        2 & 2 & $1.77{\times}10^{-3}$ & $\bf 1.77{\times}10^{-2}$ & $3.22{\times}10^{0}$ & $9.39{\times}10^{-1}$ \\
        \midrule
        3 & 3 & $2.24{\times}10^{-3}$ & $\bf 1.95{\times}10^{-2}$ & $3.32{\times}10^{0}$ & $6.45{\times}10^{0}$ \\
        \midrule
        4 & 4 & $1.86{\times}10^{-3}$ & $\bf 1.92{\times}10^{-2}$ & $6.88{\times}10^{0}$ & $8.84{\times}10^{1}$ \\
        \midrule
        5 & 5 & $1.20{\times}10^{-3}$ & $\bf 2.27{\times}10^{-2}$ & $3.36{\times}10^{1}$ & N/A \\
        \midrule
        6 & 6 & $8.47{\times}10^{-4}$ & $\bf 2.31{\times}10^{-2}$ & $7.04{\times}10^{1}$ & N/A \\
        \midrule
        \bottomrule
    \end{tabular}
    \label{table:first_order_benchmark_iterative}
    
    \vspace{+1em}
    
    \centering
    \captionsetup{justification=centering}
    \caption{Average time required for \textbf{\textcolor{blue}{iterative}} methods to reach the same ergodic metric value\\(\textbf{second-order} system dynamics).}
    \setlength{\tabcolsep}{20.0pt}
    \begin{tabular}{cccccc}
        \toprule
        \midrule
        \multirow{3}{*}{\begin{tabular}[c]{@{}c@{}}Task\\Dim.\end{tabular}} & \multirow{3}{*}{\begin{tabular}[c]{@{}c@{}}System\\Dim.\end{tabular}} & \multirow{3}{*}{\begin{tabular}[c]{@{}c@{}}Average Target\\Ergodic Metric Value\end{tabular}} & \multicolumn{3}{c}{Average Elapsed Time (second)}  \\
        \cmidrule(lr){4-6} 
        & & & \multirow{2}{*}{\begin{tabular}[c]{@{}c@{}}\textsf{Ours}\\(Iterative)\end{tabular}} & \multirow{2}{*}{\begin{tabular}[c]{@{}c@{}}\textsf{TT}\\(\textcolor{blue}{Iterative})\end{tabular}} & \multirow{2}{*}{\begin{tabular}[c]{@{}c@{}}\textsf{SMC}\\(\textcolor{blue}{Iterative})\end{tabular}} \\
        & & & & & \\
        \midrule
        \midrule
        2 & 4 & $3.06{\times}10^{-3}$ & $\bf 1.85{\times}10^{-2}$ & $3.88{\times}10^{0}$ & $1.42{\times}10^{0}$ \\
        \midrule
        3 & 6 & $3.35{\times}10^{-3}$ & $\bf 2.37{\times}10^{-2}$ & $3.79{\times}10^{0}$ & $8.61{\times}10^{0}$ \\
        \midrule
        4 & 8 & $2.12{\times}10^{-3}$ & $\bf 3.94{\times}10^{-2}$ & $7.86{\times}10^{0}$ & $1.04{\times}10^{2}$ \\
        \midrule
        5 & 10 & $2.19{\times}10^{-3}$ & $\bf 5.66{\times}10^{-2}$ & $1.46{\times}10^{1}$ & N/A \\
        \midrule
        6 & 12 & $1.13{\times}10^{-3}$ & $\bf 6.28{\times}10^{-2}$ & $3.90{\times}10^{1}$ & N/A \\
        \midrule
        \bottomrule
    \end{tabular}
    \label{table:second_order_benchmark_iterative}
\end{table*}

\begin{table*}[htbp]
    \centering
    \captionsetup{justification=centering}
    \caption{Benchmark results of the proposed method and \textbf{\textcolor{Green}{greedy}} baselines\\(\textbf{first-order} system dynamics).}
    \setlength{\tabcolsep}{18.0pt}
    \begin{tabular}{cccccc}
        \toprule
        \midrule
        \multirow{3}{*}{\begin{tabular}[c]{@{}c@{}}Task\\Dim.\end{tabular}} & \multirow{3}{*}{\begin{tabular}[c]{@{}c@{}}System\\Dim.\end{tabular}} & \multirow{3}{*}{\begin{tabular}[c]{@{}c@{}}Metrics (Average)\end{tabular}} & \multicolumn{3}{c}{Results (Average)}  \\
        \cmidrule(lr){4-6} 
        & & & \multirow{2}{*}{\begin{tabular}[c]{@{}c@{}}\textsf{Ours}\\(Iterative)\end{tabular}} & \multirow{2}{*}{\begin{tabular}[c]{@{}c@{}}\textsf{TT}\\\textit{(\textcolor{Green}{Greedy})}\end{tabular}} & \multirow{2}{*}{\begin{tabular}[c]{@{}c@{}}\textsf{SMC}\\\textit{(\textcolor{Green}{Greedy})}\end{tabular}} \\
         & & & & & \\
        \midrule
        \midrule
        \multirow{2}{*}{2} & \multirow{2}{*}{2} & Ergodic Metric & $1.77{\times}10^{-3}$ & $\bf 1.70{\times}10^{-3}$ & $2.25{\times}10^{-3}$ \\
        \cmidrule(lr){3-6} 
        & & Elapsed Time (second) & $\bf 1.77{\times}10^{-2}$ & $1.39{\times}10^{-1}$ & $1.93{\times}10^{-2}$ \\
        \midrule
        \midrule
        \multirow{2}{*}{3} & \multirow{2}{*}{3} & Ergodic Metric & $\bf 2.24{\times}10^{-3}$ & $4.24{\times}10^{-3}$ & $6.55{\times}10^{-3}$ \\
        \cmidrule(lr){3-6}
        & & Elapsed Time (second) & $\bf 1.95{\times}10^{-2}$ & $4.50{\times}10^{-1}$ & $1.01{\times}10^{-1}$ \\
        \midrule
        \midrule
        \multirow{2}{*}{4} & \multirow{2}{*}{4} & Ergodic Metric & $\bf 1.86{\times}10^{-3}$ & $3.69{\times}10^{-3}$ & $3.26{\times}10^{-3}$ \\
        \cmidrule(lr){3-6}
        & & Elapsed Time (second) & $\bf 1.92{\times}10^{-2}$ & $1.18{\times}10^{0}$ & $4.76{\times}10^{0}$ \\
        \midrule
        \midrule
        \multirow{2}{*}{5} & \multirow{2}{*}{5} & Ergodic Metric & $\bf 1.20{\times}10^{-3}$ & $4.32{\times}10^{-3}$ & N/A \\
        \cmidrule(lr){3-6}
        & & Elapsed Time (second) & $\bf 2.27{\times}10^{-2}$ & $4.03{\times}10^{0}$ & N/A \\
        \midrule
        \midrule
        \multirow{2}{*}{6} & \multirow{2}{*}{6} & Ergodic Metric & $\bf 8.47{\times}10^{-4}$ & $2.80{\times}10^{-3}$ & N/A \\
        \cmidrule(lr){3-6}
        & & Elapsed Time (second) & $\bf 2.31{\times}10^{-2}$ & $9.53{\times}10^{0}$ & N/A \\
        \midrule
        \bottomrule
    \end{tabular}
    \label{table:first_order_benchmark_greedy}

    \vspace{+1em}

    \centering
    \captionsetup{justification=centering}
    \caption{Benchmark results of the proposed method and \textbf{\textcolor{Green}{greedy}} baselines\\(\textbf{second-order} system dynamics).}
    \setlength{\tabcolsep}{18.0pt}
    \begin{tabular}{cccccc}
        \toprule
        \midrule
        \multirow{3}{*}{\begin{tabular}[c]{@{}c@{}}Task\\Dim.\end{tabular}} & \multirow{3}{*}{\begin{tabular}[c]{@{}c@{}}System\\Dim.\end{tabular}} & \multirow{3}{*}{\begin{tabular}[c]{@{}c@{}}Metrics (Average)\end{tabular}} & \multicolumn{3}{c}{Results (Average)}  \\
        \cmidrule(lr){4-6} 
        & & & \multirow{2}{*}{\begin{tabular}[c]{@{}c@{}}\textsf{Ours}\\(Iterative)\end{tabular}} & \multirow{2}{*}{\begin{tabular}[c]{@{}c@{}}\textsf{TT}\\\textit{(\textcolor{Green}{Greedy})}\end{tabular}} & \multirow{2}{*}{\begin{tabular}[c]{@{}c@{}}\textsf{SMC}\\\textit{(\textcolor{Green}{Greedy})}\end{tabular}} \\
         & & & & & \\
        \midrule
        \midrule
        \multirow{2}{*}{2} & \multirow{2}{*}{4} & Ergodic Metric & $\bf 3.06{\times}10^{-3}$ & $1.42{\times}10^{-2}$ & $1.45{\times}10^{-2}$ \\
        \cmidrule(lr){3-6} 
        & & Elapsed Time (second) & $\bf 1.85{\times}10^{-2}$ & $1.12{\times}10^{-1}$ & $2.37{\times}10^{-2}$ \\
        \midrule
        \midrule
        \multirow{2}{*}{3} & \multirow{2}{*}{6} & Ergodic Metric & $\bf 1.35{\times}10^{-3}$ & $1.45{\times}10^{-2}$ & $1.52{\times}10^{-2}$ \\
        \cmidrule(lr){3-6}
        & & Elapsed Time (second) & $\bf 2.37{\times}10^{-2}$ & $3.07{\times}10^{-1}$ & $1.08{\times}10^{-1}$ \\
        \midrule
        \midrule
        \multirow{2}{*}{4} & \multirow{2}{*}{8} & Ergodic Metric & $\bf 2.12{\times}10^{-3}$ & $1.70{\times}10^{-2}$ & $1.78{\times}10^{-2}$ \\
        \cmidrule(lr){3-6}
        & & Elapsed Time (second) & $\bf 3.94{\times}10^{-2}$ & $1.10{\times}10^{0}$ & $5.28{\times}10^{0}$ \\
        \midrule
        \midrule
        \multirow{2}{*}{5} & \multirow{2}{*}{10} & Ergodic Metric & $\bf 2.19{\times}10^{-3}$ & $1.94{\times}10^{-2}$ & N/A \\
        \cmidrule(lr){3-6}
        & & Elapsed Time (second) & $\bf 5.66{\times}10^{-2}$ & $2.40{\times}10^{0}$ & N/A \\
        \midrule
        \midrule
        \multirow{2}{*}{6} & \multirow{2}{*}{12} & Ergodic Metric & $\bf 1.13{\times}10^{-3}$ & $1.97{\times}10^{-2}$ & N/A \\
        \cmidrule(lr){3-6}
        & & Elapsed Time (second) & $\bf 6.28{\times}10^{-2}$ & $6.03{\times}10^{0}$ & N/A \\
        \midrule
        \bottomrule
    \end{tabular}
    \label{table:second_order_benchmark_greedy}
\end{table*}

\noindent\textbf{[Rationale of baseline selection] }  We compare the proposed algorithm with methods that optimize the Fourier ergodic metric with four baseline methods:
\begin{itemize}
    \item \textsf{SMC}(Greedy): The ergodic search algorithm proposed in~\cite{mathew_metrics_2011} that optimizes the Fourier ergodic metric. It is essentially a greedy receding-horizon planning algorithm with the planning horizon being infinitesimally small. 
    \item \textsf{SMC}(iterative): An iterative trajectory optimization algorithm proposed in~\cite{miller_trajectory_2013} that optimizes the Fourier ergodic metric. It follows a similar derivation as Algorithm~\ref{algo:traj_opt}, as it iteratively solves an LQR problem. 
    \item \textsf{TT}(Greedy): An algorithm that shares the same formula of \textsf{SMC}(Greedy), but it accelerates the computation of the Fourier ergodic metric through tensor-train decomposition. Proposed in~\cite{shetty_ergodic_2022}, this algorithm is the state-of-the-art fast ergodic search algorithm with a greedy receding-horizon planning formula.
    \item \textsf{TT}(Iterative): We accelerate the computation of the iterative trajectory optimization algorithm for the Fourier ergodic metric---\textsf{SMC}(Iterative)---through the same tensor-training decomposition technique used in~\cite{shetty_ergodic_2022}. This method is the state-of-the-art trajectory optimization method for ergodic search.
\end{itemize} 
We choose \textsf{SMC}(Greedy) since it is one of the most commonly used algorithms in robotic applications. For the same reason, we choose \textsf{TT}(Greedy) as it further accelerates the computation of \textsf{SMC}(Greedy), thus serving as the state-of-the-art fast ergodic search baseline. We choose \textsf{SMC}(Iterative), as well as \textsf{TT}(Iterative), since the algorithms are conceptually similar to our proposed algorithm, given both methods use the same iterative optimization scheme as in Algorithm~\ref{algo:traj_opt}. Iterative methods generate better ergodic trajectories with the same number of time steps since they optimize over the whole trajectory, while the greedy methods only myopically optimize one time step at a time. However, for the same reason, iterative methods, in general, are less computationally efficient. We do not include~\cite{abraham_ergodic_2021} for comparison because it does not generalize to Lie groups. The computation of the Fourier ergodic metric in \textsf{SMC} methods is implemented in Python with vectorization. We use the implementation from \cite{shetty_ergodic_2022} for the tensor-training accelerated Fourier ergodic metric, which is implemented with the Tensor-Train Toobox~\cite{oseledets_tensor-train_2024} with key tensor train operations implemented in Fortran with multi-thread CPU parallelization. We implement our algorithm in C++ with OpenMP~\cite{dagum_openmp_1998} for multi-thread CPU parallelization. All methods are tested on a server with an AMD 5995WX CPU. No GPU acceleration is used in the experiments. The code of our implementation is available at \url{https://sites.google.com/view/kernel-ergodic/}.

\noindent\textbf{[Experiment design] } We test each of the four baseline methods and the proposed kernel ergodic search method across 2-dimensional to 6-dimensional spaces, which cover the majority of the state spaces used in robotics. Each search space is defined as a squared space, where each dimension has a boundary of $[0, 1]$ meters. For each number of dimensions, we randomize 100 test trials, with each trial consisting of a randomized three-mode Gaussian-mixture distribution (with full covariance) as the target distribution. The means of each Gaussian mixture distribution are sampled uniformly within the search space, and the covariance matrices are sampled uniformly from the space of positive definite matrices using the approach from~\cite{shetty_ergodic_2022}, with the diagonal entries varying from $0.01$ to $0.02$. In each trial, all the algorithms will explore the same target distribution with the same initial position; all the iterative methods (including ours) will start with the same initial trajectory generated from the proposed bootstrap approach (see Section~\ref{subsec:bootstrap}) and will run for a same number of iterations. We test all the algorithms with both the first-order and second-order point-mass dynamical systems. All methods have a time horizon of 200 time steps with a time step interval being $0.1$ second.

\noindent\textbf{[Metric selection] } The benchmark takes two measurements: (1) the \emph{Fourier ergodic metric} of the generated trajectory from each algorithm and (2) the computation time efficiency of each algorithm. We choose the Fourier ergodic metric as it is ubiquitous for all existing ergodic search methods and the optimization objective for all four baselines. Our proposed method optimizes the kernel ergodic metric. Still, we have shown that it is asymptotic and consistent with the Fourier ergodic metric, making the Fourier ergodic metric a suitable metric to use in evaluating our method as well. For the greedy baselines, we measure the elapsed time of the single-shot trajectory generation process and the Fourier ergodic metric of the final trajectory. For iterative baselines and our algorithm, we compute the Fourier ergodic metric of our proposed method at convergence and measure the time each method takes to reach the same level of ergodicity. We measure the computation time efficiency for the iterative methods this way because the primary constraint for all iterative methods is not the quality of the ergodic trajectory, as all iterative methods will eventually converge to (at least) a local optimum of the ergodic metric, but instead to generate trajectory with sufficient ergodicity within limited time. 

\noindent\textbf{[Results] } Table~\ref{table:first_order_benchmark_iterative} and Table~\ref{table:second_order_benchmark_iterative} show the averaged time required for each iterative method (including ours) to reach the same level of Fourier ergodic metric value from 2D to 6D space and across first-order and second-order system dynamics. We can see the proposed method is at least two orders of magnitude faster than the baselines, particularly when the search space dimension is higher than three and with second-order system dynamics. We evaluate the \textsf{SMC}(iterative) baseline only up to 4-dimensional space as the memory consumption of computing the Fourier ergodic metric beyond 4D space exceeds the computer's limit, leading to prohibitively long computation time (we record a memory consumption larger than 120 GB and the elapsed time longer than 6 minutes for a single iteration in a 5D space; the excessive computation resource consumption of the Fourier ergodic metric in high-dimensional spaces is discussed in~\cite{shetty_ergodic_2022}). Figure~\ref{fig:scalability_plot} further shows the better scalability of the proposed method, as it exhibits a linear time complexity in search space dimension while the \textsf{SMC}(iterative) method exhibits an exponential time complexity and the \textsf{TT}(iterative) method exhibits a super-linear time complexity and much slower speed with the same computation resources. Table~\ref{table:first_order_benchmark_greedy} and Table~\ref{table:second_order_benchmark_greedy} show the comparison between the proposed and non-iterative greedy baseline methods. Despite the improvement in computation efficiency of the non-iterative baselines, the proposed method is still at least two orders of magnitude faster and generates trajectories with better ergodicity. Lastly, Figure~\ref{fig:sim_traj_6d} shows an example ergodic trajectory generated by our method in a 6-dimensional space with second-order system dynamics.

\subsection{Ergodic Coverage for Peg-in-Hole Insertion in SE(3)} 

\noindent\textbf{[Motivation] } Given the complexity of robotic manipulation tasks, using human demonstration for robots to acquire manipulation skills is becoming increasingly popular~\cite{ravichandar_recent_2020}, in particular for robotic insertion tasks, which are critical for applications including autonomous assembly~\cite{wu_prim-lafd_2023} and household robots~\cite{zhang_vision-based_2023}. Most approaches for acquiring insertion skills from demonstrations are learning-based, where the goal is to learn a control policy~\cite{wen_you_2022} or task objective from the demonstrations~\cite{englert_learning_2018}. One common strategy to learn insertion skills from demonstration is to learn motion primitives, such as dynamic movement primitives (DMPs), from the demonstrations as control policies, which could dramatically reduce the search space for learning~\cite{saveriano_dynamic_2023}. Furthermore, to address the potential mismatch between the demonstration and the task (e.g., the location of insertion during task execution may differ from the demonstration), the learned policies are often explicitly enhanced with local exploration policies, for example, through hand-crafted exploratory motion primitive~\cite{wu_prim-lafd_2023}, programmed compliance control with torque-force feedback~\cite{jha_imitation_2022} and residual correction policy~\cite{davchev_residual_2022}. Another common strategy is to use human demonstrations to bootstrap reinforcement learning (RL) training in simulation~\cite{luo_robust_2021,ahn_robotic_2023,guo_reinforcement_2023}, where the demonstrations could address the sparsity of the reward function, thus accelerate the convergence of the policy. Instead of using learning-from-demonstration methods, our motivation is to provide an alternative learning-free framework to obtain manipulation skills from human demonstrations. We formulate the peg-in-hole insertion task as a coverage problem, where the robot must find the successful insertion configuration using the human demonstration as the prior distribution. We show that combining this search-based problem formulation with ergodic coverage leads to reliable insertion performance while avoiding the typical limitations of learning-from-demonstration methods, such as the limited demonstration data, limited sensor measurement, and the need for additional offline training. Nevertheless, each new task attempt could be incorporated into a learning workflow.

\begin{figure}[t]
    \centering
    \begin{subfigure}[b]{0.24\textwidth}
        \includegraphics[width=\textwidth]{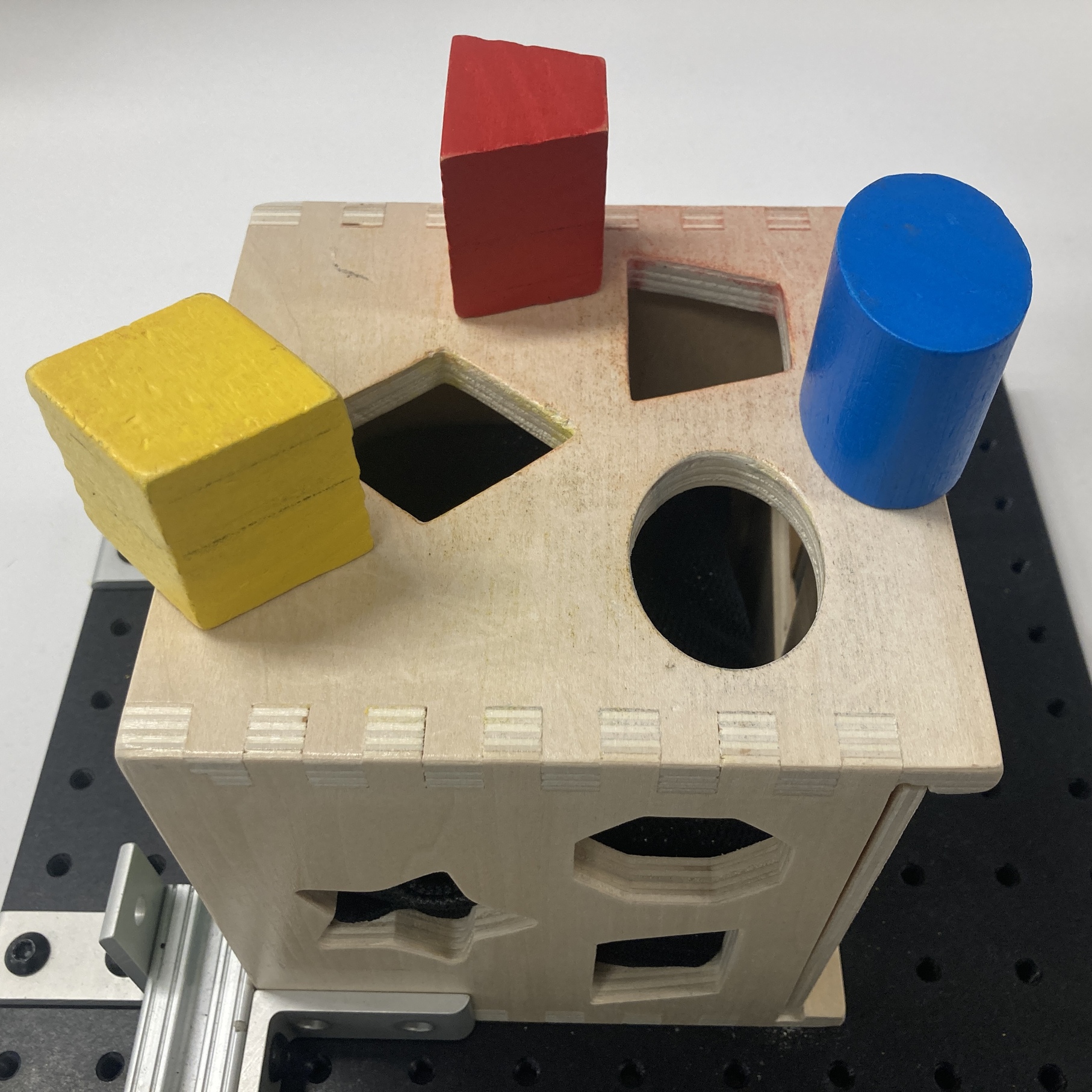}
    \end{subfigure}
    \hfill
    \begin{subfigure}[b]{0.24\textwidth}
        \includegraphics[width=\textwidth]{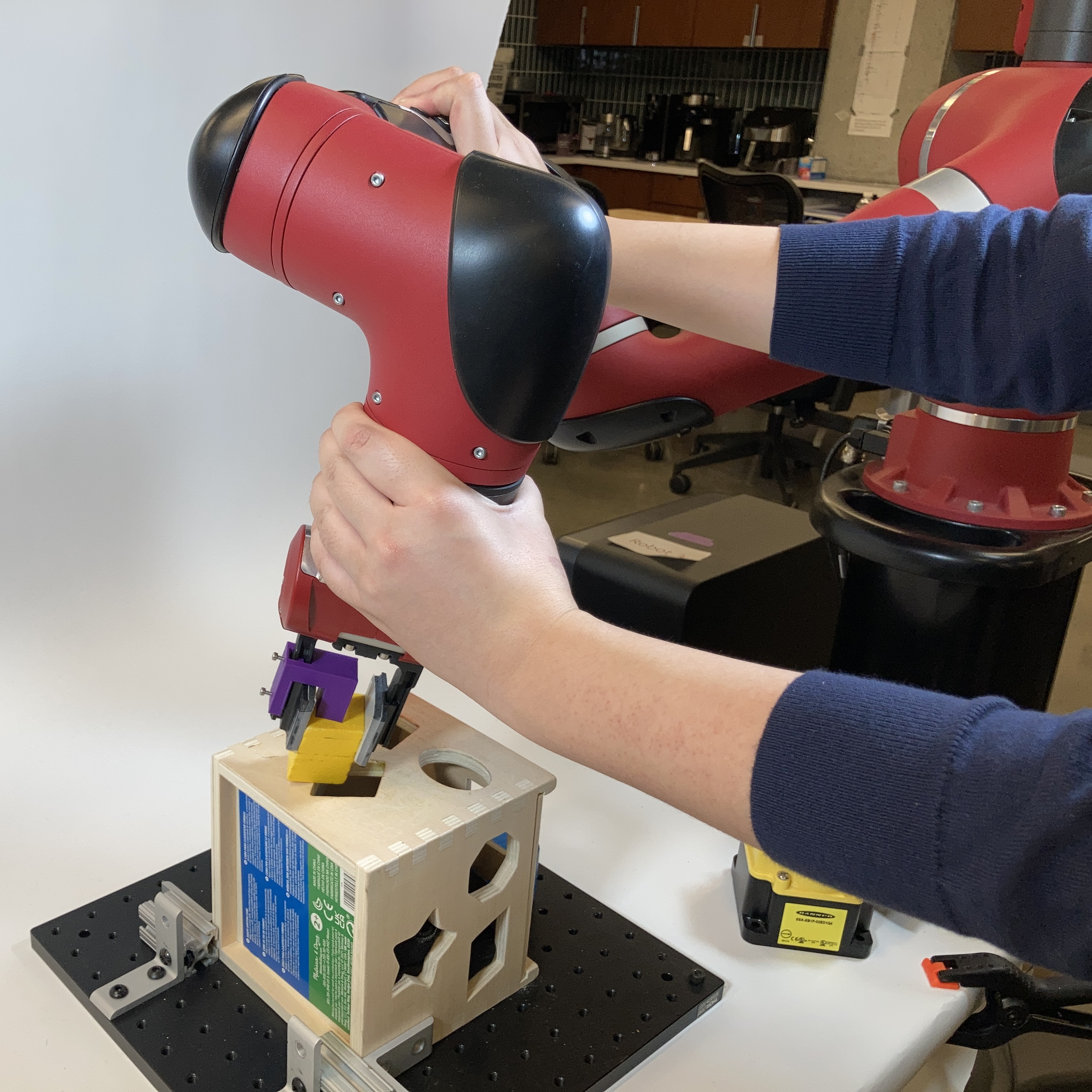}
    \end{subfigure}
    \caption{Setup of the hardware experiment.}
    \label{fig:sawyer_setup}
  
    \vspace{+1em}
  
    \centering
    \includegraphics[width=0.49\textwidth]{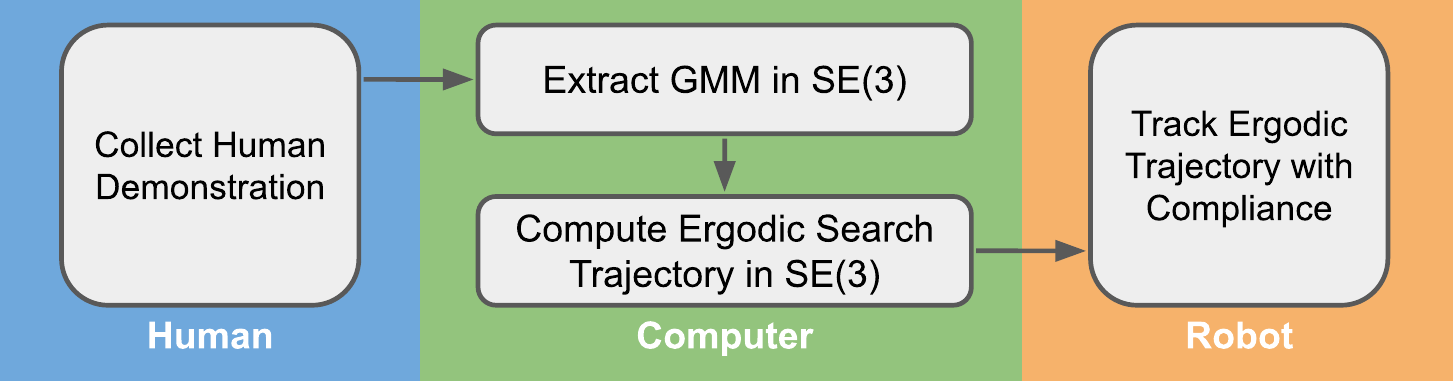}
    \caption{System diagram of acquiring insertion skill from human demonstration.}
    \label{fig:sawyer_system_diagram}
    \vspace{-1em}
\end{figure}

\noindent\textbf{[Task design] } In this task, the robot needs to find successful insertion configurations for cubes with three different geometries from a common shape sorting toy (see Figure~\ref{fig:sawyer_setup}). For each object of interest, a 30-second-long kinesthetic teaching demonstration is conducted, with a human moving the end-effector around the hole of the corresponding shape. The end-effector poses are recorded at 10 Hz in SE(3), providing a 300-time step trajectory recording as the only data available for the robot to acquire the insertion skill. During the task execution, the robot needs to generate insertion motions from a randomly initialized position within the same number of time steps as the demonstration (300 time steps). Furthermore, to demonstrate the method’s robustness to the quality of the demonstration, the demonstrations in this test do not contain successful insertions but only configurations around the corresponding hole, such as what someone attempting the task might do even if they were unsuccessful. Such insufficient demonstrations make it necessary for the robot to adapt beyond repeating the exact human demonstration provided. Some approaches attempt this adaptation through learning, whereas here adaptation is formulated as state space coverage ``near'' the demonstrated distribution of states.

\begin{figure*}[htbp]
    \centering
    \includegraphics[width=0.99\textwidth]{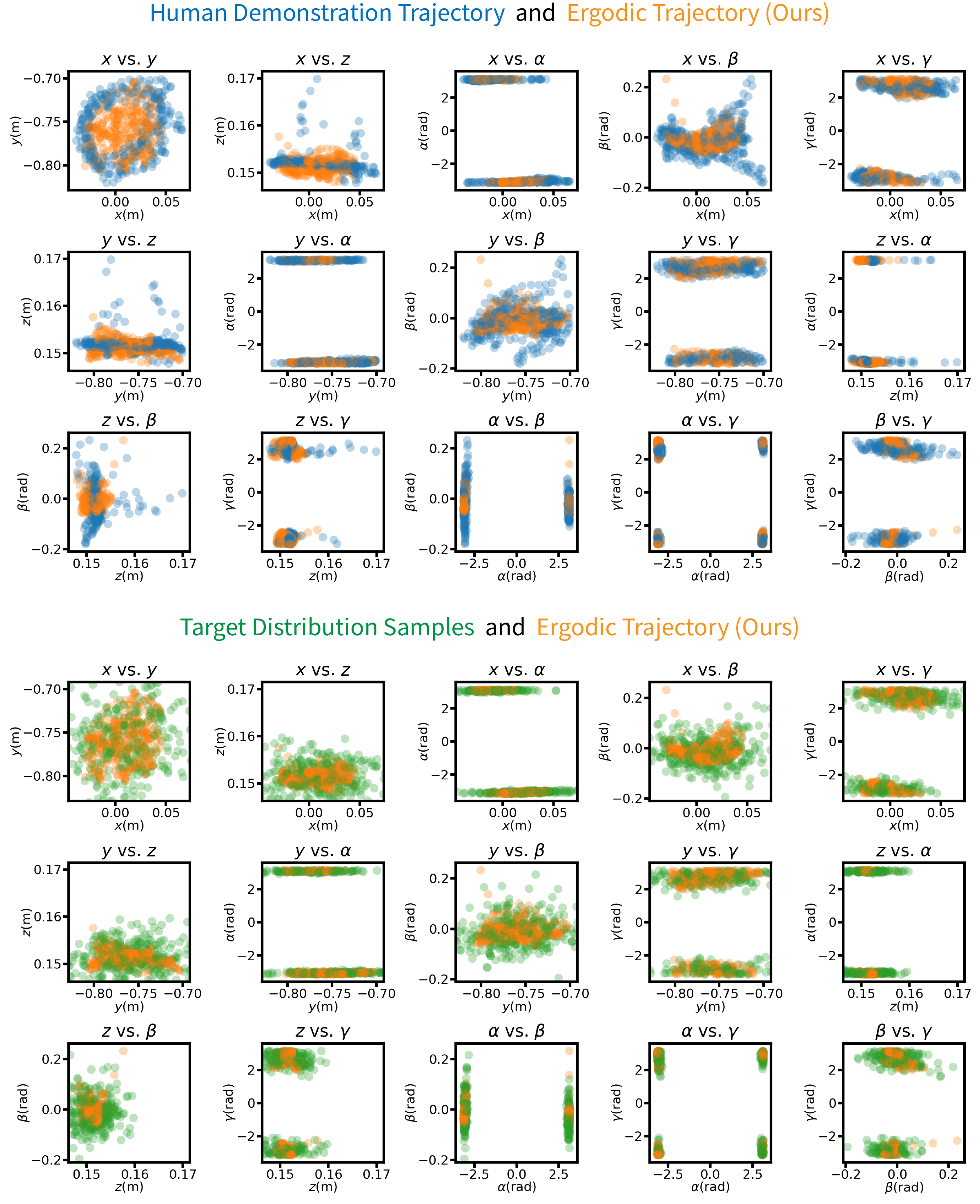}
    \caption{Trajectories from one of the hardware test trials. All trajectories are represented in the space of SE(3) and converted to the coordinate system of 3D Euclidean space $\{x, y, z\}$ for position and Euler angles $\{\alpha, \beta, \gamma\}$ for orientation. The orientations included in human demonstration lie at the boundary of the principle interval $-\pi$ and $\pi$, thus exhibiting large discontinuity in the Euler angle coordinate. The proposed algorithm's capability to directly reason over the Lie group SE(3) inherently overcomes this issue and successfully generates continuous trajectories to cover the distribution of human demonstration.}
    \label{fig:sawyer_traj_overlap}
\end{figure*}

\noindent\textbf{[Implementation details] } We use a Sawyer robot arm for the experiment. Our approach is to generate an ergodic coverage trajectory using the human demonstration as the target distribution, assuming the successful insertion configuration resides within the distribution that governs the human demonstration. The target distribution is modeled as a Lie group Gaussian-mixture (GMM) distribution, which is computed using the expectation maximization (EM) algorithm from the human demonstration. Since the demonstration does not include successful insertion, the target GMM distribution has a height (z-axis value) higher than the box's surface. Thus, we decrease the z-axis value of the GMM means for $2cm$. After the ergodic search trajectory is generated with the given target GMM distribution, the robot tracks the trajectory with online position feedback with waypoint-based tracking. We enable the force compliance feature on the robot arm~\cite{noauthor_arm_nodate}, which ensures safety for both the robot and the object during execution. No other sensor feedback, such as visual or tactile sensing, is used. A system overview is shown in the diagram in Figure~\ref{fig:sawyer_system_diagram}. Note that the waypoint-based control moves the end-effector at a slower speed than the human demonstration; thus, even if the executed search trajectory has the same number of time steps as the demonstration, it would take the robot longer real-world time to execute the trajectory. An end-effector insertion configuration is considered successful when both of the following criteria are met: (1) the end-effector's height, measured through the z-axis value, is near the exact height of the box surface, and (2) the force feedback from the end-effector detects no force along the z-axis. Meeting the first criterion means the end-effector reaches the necessary height for a possible insertion. Meeting the second criterion means the cube goes freely through the hole; it rules out the false-positive scenario where the end-effector forces the part of the cube through the hole when the cube and hole are not aligned. Lastly, the covariance matrices for both the Gaussian mixture model and the kernel function across the tests have sufficiently small eigenvalues, such that the compactness of the SE(3) group does not affect the Gaussian distribution formula, as mentioned in Remark~\ref{remark:compact_gaussian}. The demonstration dataset is available at \url{https://sites.google.com/view/kernel-ergodic/}.

\noindent\textbf{[Results] } We compare our method with a baseline method that repeats the demonstration. We test three objects in total, the shapes being rhombus, trapezoid, and ellipse. A total of 20 trials are conducted for each object; in each trial, we generate a new demonstration shared by both our method and the baseline. Both methods also share the maximum amount of time steps allowed for the insertion, which is the same as the demonstration. We measure the success rate of each method and report the average time steps required to find the successful insertion. Table~\ref{table:insertion_success_rate} shows the success rate for the insertion task across three objects within a limited search time of 300 time steps. The proposed ergodic search method has a success rate higher than or equal to $80\%$ across all three objects, while the baseline method only has a success rate up to $10\%$ across the objects. The baseline method does not have a $0\%$ success rate because of the noise within the motion from the force compliance feature during trajectory tracking. Table~\ref{table:insertion_time_required} shows the average time steps required for successful insertion, where we can see the proposed method can find a successful insertion strategy in SE(3) with significantly less time than the demonstration. Figure~\ref{fig:sawyer_traj_overlap} further shows the end-effector trajectory from the human demonstration and the resulting ergodic search trajectory, as well as how the SE(3) reasoning capability of the proposed algorithm overcomes the Euler angle discontinuity in the human demonstration.

\begin{table}[t!]
    \centering
    \captionsetup{justification=centering}
    \caption{Success rate of hardware insertion test\\(limited search time).}
    \setlength{\tabcolsep}{7.0pt}
    \begin{tabular}{cccc}
        \toprule
        \multirow{2}{*}{\begin{tabular}[c]{@{}c@{}}\textbf{Strategy}\end{tabular}} & \multicolumn{3}{c}{\textbf{Success Rate} (20 trials per object)}  \\
        \cmidrule(lr){2-4} 
        & Rhombus & Trapezoid & Ellipse \\
        \midrule 
        \textbf{Ours} & $80\%$ (16/20) & $80\%$ (16/20) & $90\%$ (18/20) \\
        \midrule 
        Naive & $10\%$ (2/20) & $10\%$ (2/20) & $10\%$ (2/20) \\
        \bottomrule
    \end{tabular}
    \label{table:insertion_success_rate}

    \vspace{+1em}
    
    \centering
    \captionsetup{justification=centering}
    \caption{Average time steps for successful insertion (limited search time).}
    \setlength{\tabcolsep}{5.0pt}
    \begin{tabular}{cccc}
        \toprule
        \multirow{2}{*}{\begin{tabular}[c]{@{}c@{}}\textbf{Strategy}\end{tabular}} & \multicolumn{3}{c}{\textbf{Average Time Steps} (20 trials per object)}  \\
        \cmidrule(lr){2-4} 
        & Rhombus & Trapezoid & Ellipse \\
        \midrule 
        \textbf{Ours} & $106.81{\pm}78.60$ & $128.44{\pm}73.81$ & $103.33{\pm}61.25$ \\
        \midrule 
        Naive & $187.50{\pm}27.50$ & $58.50{\pm}57.50$ & $31.50{\pm}30.50$ \\
        \bottomrule
    \end{tabular}
    \label{table:insertion_time_required}
\end{table}

\begin{figure}[htbp]
    \centering
    \includegraphics[width=0.49\textwidth]{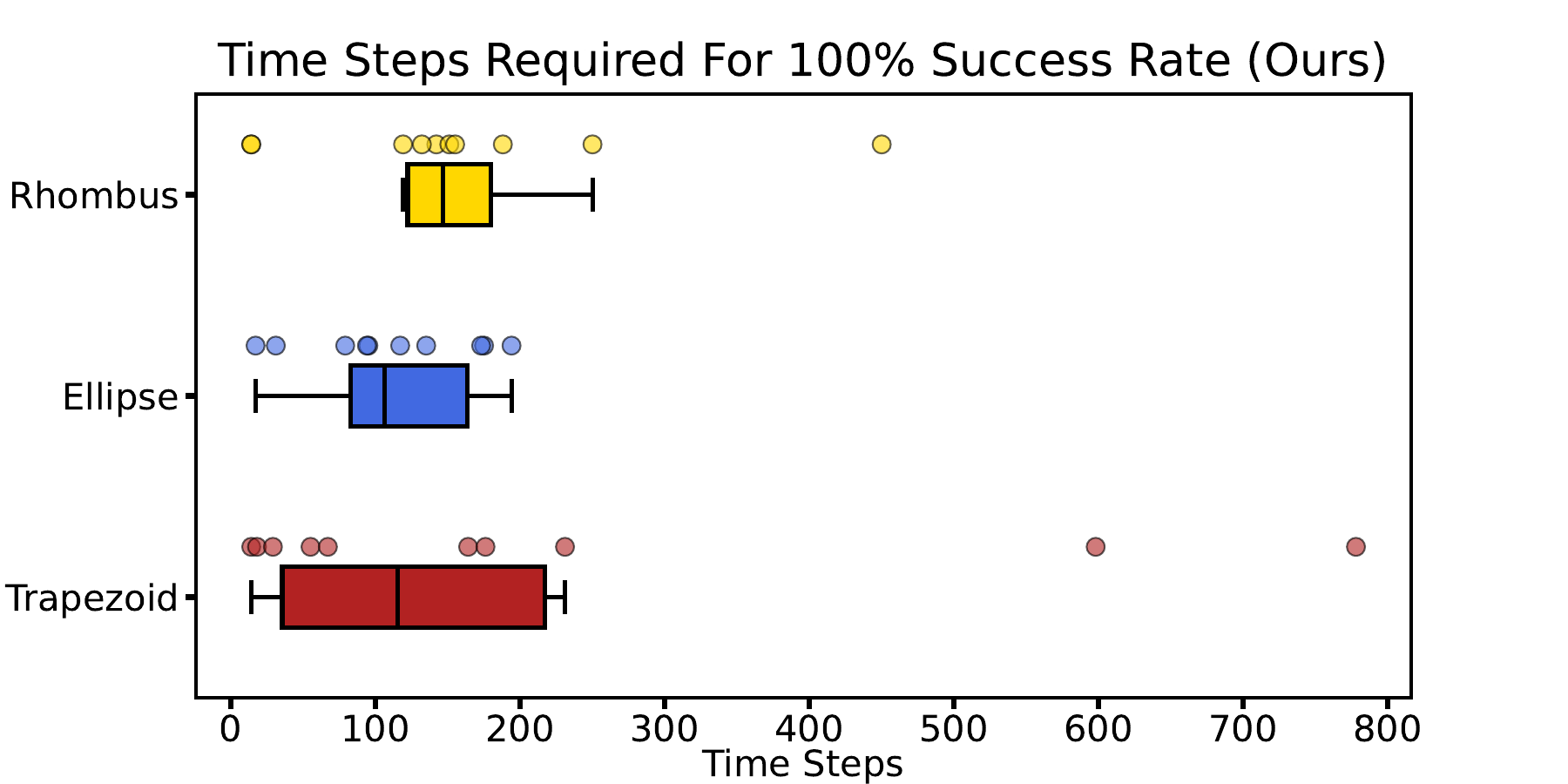}
    \caption{Time steps required for $100\%$ successful insertion with the proposed algorithm.}
    \label{fig:insertion_time_step_plot}
    \vspace{-1em}
\end{figure}

\noindent\textbf{[Asymptotic coverage] } We further demonstrate the asymptotic coverage property of ergodic search, with which the robot is guaranteed to find a successful insertion strategy given enough time, so long as the successful insertion configuration resides in the target distribution. Instead of limiting the search time to 300 time steps, we conduct 10 additional trials on each object (30 in total) with unlimited search time. Our method finds a successful insertion strategy in all 30 trials ($100\%$ success rate). We report the time steps needed for $100\%$ success rate in Figure~\ref{fig:insertion_time_step_plot}.


\section{Conclusion and Discussion} \label{sec:conclusion}

This work introduces a new ergodic search method with significantly improved computation efficiency and generalizability across Euclidean space and Lie groups. Our first contribution is introducing the kernel ergodic metric, which is asymptotically consistent with the Fourier ergodic metric but has better scalability to higher dimensional spaces. Our second contribution is an efficient optimal control method. Combining the kernel ergodic metric with the proposed optimal control method generates ergodic trajectories at least two orders of magnitude faster than the state-of-the-art method. 

We demonstrate the proposed ergodic search method through a peg-in-hole insertion task. We formulate the task as an ergodic coverage problem using a 30-second-long human demonstration as the target distribution. We demonstrate that the asymptotic coverage property of ergodic search leads to a $100\%$ success rate in this task, so long as the success insertion configuration resides within the target distribution. Our framework serves as an alternative approach to learning-from-demonstration methods.

Since our formula is based on kernel functions, it can be flexibly extended with other kernel functions for different tasks. One potential extension is to use the non-stationary attentive kernels~\cite{chen_ak_2022}, which are shown to be more effective in information-gathering tasks compared to the squared exponential kernel used in this work. The trajectory optimization-based formula means the proposed framework could be integrated into reinforcement learning (RL) with techniques such as guided policy search~\cite{levine_guided_2013}. The proposed framework can also be further improved. The evaluation of the proposed metric can be accelerated by exploiting the spatial sparsity of the kernel function evaluation within the trajectory.

{
\section*{Acknowledgments}
The authors would like to acknowledge Allison Pinosky, Davin Landry, and Sylvia Tan for their contributions to the hardware experiment. This material is supported by the Honda Research Institute Grant HRI-001479 and the National Science Foundation Grant CNS-2237576. Any opinions, findings, conclusions, or recommendations expressed in this material are those of the authors and do not necessarily reflect the views of the aforementioned institutions.
}

\bibliographystyle{IEEEtran}
\bibliography{references}

\begin{thebibliography}{10}
\providecommand{\url}[1]{#1}
\csname url@samestyle\endcsname
\providecommand{\newblock}{\relax}
\providecommand{\bibinfo}[2]{#2}
\providecommand{\BIBentrySTDinterwordspacing}{\spaceskip=0pt\relax}
\providecommand{\BIBentryALTinterwordstretchfactor}{4}
\providecommand{\BIBentryALTinterwordspacing}{\spaceskip=\fontdimen2\font plus
\BIBentryALTinterwordstretchfactor\fontdimen3\font minus
  \fontdimen4\font\relax}
\providecommand{\BIBforeignlanguage}[2]{{%
\expandafter\ifx\csname l@#1\endcsname\relax
\typeout{** WARNING: IEEEtran.bst: No hyphenation pattern has been}%
\typeout{** loaded for the language `#1'. Using the pattern for}%
\typeout{** the default language instead.}%
\else
\language=\csname l@#1\endcsname
\fi
#2}}
\providecommand{\BIBdecl}{\relax}
\BIBdecl

\bibitem{murphy_human-robot_2004}
R.~Murphy, ``Human-robot interaction in rescue robotics,'' \emph{IEEE
  Transactions on Systems, Man, and Cybernetics, Part C (Applications and
  Reviews)}, vol.~34, no.~2, pp. 138--153, May 2004, conference Name: IEEE
  Transactions on Systems, Man, and Cybernetics, Part C (Applications and
  Reviews).

\bibitem{shah_multidrone_2020}
K.~Shah, G.~Ballard, A.~Schmidt, and M.~Schwager, ``Multidrone aerial surveys
  of penguin colonies in {Antarctica},'' \emph{Science Robotics}, vol.~5,
  no.~47, p. eabc3000, Oct. 2020, publisher: American Association for the
  Advancement of Science.

\bibitem{abraham_decentralized_2018}
I.~Abraham and T.~D. Murphey, ``Decentralized {Ergodic} {Control}:
  {Distribution}-{Driven} {Sensing} and {Exploration} for {Multiagent}
  {Systems},'' \emph{IEEE Robotics and Automation Letters}, vol.~3, no.~4, pp.
  2987--2994, Oct. 2018, conference Name: IEEE Robotics and Automation Letters.

\bibitem{shetty_ergodic_2022}
S.~Shetty, J.~Silvério, and S.~Calinon, ``Ergodic {Exploration} {Using}
  {Tensor} {Train}: {Applications} in {Insertion} {Tasks},'' \emph{IEEE
  Transactions on Robotics}, vol.~38, no.~2, pp. 906--921, Apr. 2022.

\bibitem{abraham_active_2019}
I.~Abraham and T.~D. Murphey, ``Active {Learning} of {Dynamics} for
  {Data}-{Driven} {Control} {Using} {Koopman} {Operators},'' \emph{IEEE
  Transactions on Robotics}, vol.~35, no.~5, pp. 1071--1083, Oct. 2019.

\bibitem{prabhakar_mechanical_2022}
A.~Prabhakar and T.~Murphey, ``\BIBforeignlanguage{en}{Mechanical intelligence
  for learning embodied sensor-object relationships},''
  \emph{\BIBforeignlanguage{en}{Nature Communications}}, vol.~13, no.~1, p.
  4108, Jul. 2022, number: 1 Publisher: Nature Publishing Group.

\bibitem{mathew_metrics_2011}
G.~Mathew and I.~Mezić, ``\BIBforeignlanguage{en}{Metrics for ergodicity and
  design of ergodic dynamics for multi-agent systems},''
  \emph{\BIBforeignlanguage{en}{Physica D: Nonlinear Phenomena}}, vol. 240, no.
  4-5, pp. 432--442, Feb. 2011.

\bibitem{petersen_ergodic_1989}
K.~E. Petersen, \emph{\BIBforeignlanguage{en}{Ergodic {Theory}}}.\hskip 1em
  plus 0.5em minus 0.4em\relax Cambridge University Press, Nov. 1989,
  google-Books-ID: is\_LCgAAQBAJ.

\bibitem{mathew_multiscale_2005}
G.~Mathew, I.~Mezić, and L.~Petzold, ``A multiscale measure for mixing,''
  \emph{Physica D: Nonlinear Phenomena}, vol. 211, no.~1, pp. 23--46, Nov.
  2005.

\bibitem{chen_tuning_2020}
C.~Chen, T.~D. Murphey, and M.~A. MacIver, ``\BIBforeignlanguage{en}{Tuning
  movement for sensing in an uncertain world},''
  \emph{\BIBforeignlanguage{en}{eLife}}, vol.~9, p. e52371, Sep. 2020.

\bibitem{sun_scale-invariant_2022}
M.~Sun, A.~Pinosky, I.~Abraham, and T.~Murphey,
  ``\BIBforeignlanguage{en}{Scale-{Invariant} {Fast} {Functional}
  {Registration}},'' in \emph{\BIBforeignlanguage{en}{Robotics {Research}}},
  ser. Springer {Proceedings} in {Advanced} {Robotics}.\hskip 1em plus 0.5em
  minus 0.4em\relax Springer International Publishing, 2022.

\bibitem{hauser_projection_2002}
J.~Hauser, ``A {Projection} {Operator} {Approach} to the {Optimization} of
  {Trajectory} {Functionals},'' \emph{IFAC Proceedings Volumes}, vol.~35,
  no.~1, pp. 377--382, Jan. 2002.

\bibitem{miller_trajectory_2013-1}
L.~M. Miller and T.~D. Murphey, ``Trajectory optimization for continuous
  ergodic exploration,'' in \emph{2013 {American} {Control} {Conference}},
  2013, pp. 4196--4201.

\bibitem{miller_trajectory_2013}
------, ``Trajectory optimization for continuous ergodic exploration on the
  motion group {SE}(2),'' in \emph{52nd {IEEE} {Conference} on {Decision} and
  {Control}}, 2013, pp. 4517--4522.

\bibitem{abraham_ergodic_2021}
I.~Abraham, A.~Prabhakar, and T.~D. Murphey, ``An {Ergodic} {Measure} for
  {Active} {Learning} {From} {Equilibrium},'' \emph{IEEE Transactions on
  Automation Science and Engineering}, vol.~18, no.~3, pp. 917--931, Jul. 2021.

\bibitem{tang_note_2023}
Y.~Tang, ``A {Note} on {Monte} {Carlo} {Integration} in {High} {Dimensions},''
  \emph{The American Statistician}, vol.~0, no.~0, pp. 1--7, 2023, publisher:
  Taylor \& Francis \_eprint: https://doi.org/10.1080/00031305.2023.2267637.

\bibitem{walters_introduction_2000}
P.~Walters, \emph{\BIBforeignlanguage{en}{An {Introduction} to {Ergodic}
  {Theory}}}.\hskip 1em plus 0.5em minus 0.4em\relax Springer Science \&
  Business Media, Oct. 2000, google-Books-ID: eCoufOp7ONMC.

\bibitem{mavrommati_real-time_2018}
A.~Mavrommati, E.~Tzorakoleftherakis, I.~Abraham, and T.~D. Murphey,
  ``Real-{Time} {Area} {Coverage} and {Target} {Localization} {Using}
  {Receding}-{Horizon} {Ergodic} {Exploration},'' \emph{IEEE Transactions on
  Robotics}, vol.~34, no.~1, pp. 62--80, Feb. 2018.

\bibitem{kalinowska_ergodic_2021}
A.~Kalinowska, A.~Prabhakar, K.~Fitzsimons, and T.~Murphey, ``Ergodic
  imitation: {Learning} from what to do and what not to do,'' in \emph{2021
  {IEEE} {International} {Conference} on {Robotics} and {Automation} ({ICRA})},
  May 2021, pp. 3648--3654, iSSN: 2577-087X.

\bibitem{ehlers_imitating_2019}
D.~Ehlers, M.~Suomalainen, J.~Lundell, and V.~Kyrki, ``Imitating {Human}
  {Search} {Strategies} for {Assembly},'' in \emph{2019 {International}
  {Conference} on {Robotics} and {Automation} ({ICRA})}, May 2019, pp.
  7821--7827, iSSN: 2577-087X.

\bibitem{lerch_safety-critical_2023}
C.~Lerch, D.~Dong, and I.~Abraham, ``Safety-{Critical} {Ergodic} {Exploration}
  in {Cluttered} {Environments} via {Control} {Barrier} {Functions},'' in
  \emph{2023 {IEEE} {International} {Conference} on {Robotics} and {Automation}
  ({ICRA})}, May 2023, pp. 10\,205--10\,211.

\bibitem{ren_pareto-optimal_2023}
Z.~Ren, A.~K. Srinivasan, B.~Vundurthy, I.~Abraham, and H.~Choset, ``A
  {Pareto}-{Optimal} {Local} {Optimization} {Framework} for {Multiobjective}
  {Ergodic} {Search},'' \emph{IEEE Transactions on Robotics}, vol.~39, no.~5,
  pp. 3452--3463, Oct. 2023.

\bibitem{dong_time_2023}
D.~Dong, H.~Berger, and I.~Abraham, ``\BIBforeignlanguage{en}{Time {Optimal}
  {Ergodic} {Search}},'' in \emph{\BIBforeignlanguage{en}{Robotics: {Science}
  and {Systems} {XIX}}}.\hskip 1em plus 0.5em minus 0.4em\relax Robotics:
  Science and Systems Foundation, Jul. 2023.

\bibitem{mack_fundamental_2007}
C.~Mack, \emph{\BIBforeignlanguage{en}{Fundamental {Principles} of {Optical}
  {Lithography}: {The} {Science} of {Microfabrication}}}.\hskip 1em plus 0.5em
  minus 0.4em\relax John Wiley \& Sons, Dec. 2007.

\bibitem{lighthill_introduction_1958}
S.~M.~J. Lighthill, \emph{\BIBforeignlanguage{en}{An {Introduction} to
  {Fourier} {Analysis} and {Generalised} {Functions}}}.\hskip 1em plus 0.5em
  minus 0.4em\relax Cambridge University Press, 1958.

\bibitem{rudin_functional_1991}
W.~Rudin, \emph{\BIBforeignlanguage{en}{Functional {Analysis}}}.\hskip 1em plus
  0.5em minus 0.4em\relax McGraw-Hill, 1991.

\bibitem{strichartz_guide_1994}
R.~S. Strichartz, \emph{\BIBforeignlanguage{en}{A {Guide} {To} {Distribution}
  {Theory} {And} {Fourier} {Transforms}}}.\hskip 1em plus 0.5em minus
  0.4em\relax World Scientific Publishing Company, 1994.

\bibitem{cohen-tannoudji_quantum_1977}
C.~Cohen-Tannoudji, \emph{\BIBforeignlanguage{eng;fre}{Quantum
  mechanics}}.\hskip 1em plus 0.5em minus 0.4em\relax New York: Wiley, 1977.

\bibitem{rasmussen_gaussian_2006}
C.~E. Rasmussen and C.~K.~I. Williams, \emph{\BIBforeignlanguage{en}{Gaussian
  processes for machine learning}}, ser. Adaptive computation and machine
  learning.\hskip 1em plus 0.5em minus 0.4em\relax Cambridge, Mass: MIT Press,
  2006, oCLC: ocm61285753.

\bibitem{vaart_asymptotic_1998}
A.~W. v.~d. Vaart, \emph{Asymptotic {Statistics}}, ser. Cambridge {Series} in
  {Statistical} and {Probabilistic} {Mathematics}.\hskip 1em plus 0.5em minus
  0.4em\relax Cambridge: Cambridge University Press, 1998.

\bibitem{miller_ergodic_2016}
L.~M. Miller, Y.~Silverman, M.~A. MacIver, and T.~D. Murphey, ``Ergodic
  {Exploration} of {Distributed} {Information},'' \emph{IEEE Transactions on
  Robotics}, vol.~32, no.~1, pp. 36--52, Feb. 2016.

\bibitem{nocedal_numerical_2006}
J.~Nocedal and S.~J. Wright, \emph{\BIBforeignlanguage{en}{Numerical
  optimization}}, 2nd~ed., ser. Springer series in operations research.\hskip
  1em plus 0.5em minus 0.4em\relax New York: Springer, 2006, oCLC: ocm68629100.

\bibitem{rosenkrantz_analysis_1977}
D.~J. Rosenkrantz, R.~E. Stearns, and P.~M. Lewis, II, ``An {Analysis} of
  {Several} {Heuristics} for the {Traveling} {Salesman} {Problem},'' \emph{SIAM
  Journal on Computing}, vol.~6, no.~3, pp. 563--581, Sep. 1977, publisher:
  Society for Industrial and Applied Mathematics.

\bibitem{choset_principles_2005}
H.~Choset, K.~M. Lynch, S.~Hutchinson, G.~A. Kantor, and W.~Burgard,
  \emph{\BIBforeignlanguage{en}{Principles of {Robot} {Motion}: {Theory},
  {Algorithms}, and {Implementations}}}.\hskip 1em plus 0.5em minus 0.4em\relax
  MIT Press, May 2005.

\bibitem{chirikjian_stochastic_2009}
G.~S. Chirikjian, \emph{\BIBforeignlanguage{en}{Stochastic {Models},
  {Information} {Theory}, and {Lie} {Groups}, {Volume} 1: {Classical} {Results}
  and {Geometric} {Methods}}}.\hskip 1em plus 0.5em minus 0.4em\relax Springer
  Science \& Business Media, Sep. 2009.

\bibitem{lynch_modern_2017}
K.~M. Lynch and F.~C. Park, \emph{\BIBforeignlanguage{en}{Modern
  {Robotics}}}.\hskip 1em plus 0.5em minus 0.4em\relax Cambridge University
  Press, May 2017, google-Books-ID: 5NzFDgAAQBAJ.

\bibitem{sola_micro_2021}
J.~Solà, J.~Deray, and D.~Atchuthan, ``A micro {Lie} theory for state
  estimation in robotics,'' Dec. 2021, arXiv:1812.01537 [cs].

\bibitem{boumal_introduction_2023}
N.~Boumal, \emph{An {Introduction} to {Optimization} on {Smooth}
  {Manifolds}}.\hskip 1em plus 0.5em minus 0.4em\relax Cambridge: Cambridge
  University Press, 2023.

\bibitem{fan_online_2016}
T.~Fan and T.~Murphey, ``\BIBforeignlanguage{en}{Online {Feedback} {Control}
  for {Input}-{Saturated} {Robotic} {Systems} on {Lie} {Groups}},'' in
  \emph{\BIBforeignlanguage{en}{Robotics: {Science} and {Systems}
  {XII}}}.\hskip 1em plus 0.5em minus 0.4em\relax Robotics: Science and Systems
  Foundation, 2016.

\bibitem{yunfeng_wang_error_2006}
{Yunfeng Wang} and G.~Chirikjian, ``\BIBforeignlanguage{en}{Error propagation
  on the {Euclidean} group with applications to manipulator kinematics},''
  \emph{\BIBforeignlanguage{en}{IEEE Transactions on Robotics}}, vol.~22,
  no.~4, pp. 591--602, Aug. 2006.

\bibitem{wang_nonparametric_2008}
Y.~Wang and G.~S. Chirikjian, ``\BIBforeignlanguage{en}{Nonparametric
  {Second}-order {Theory} of {Error} {Propagation} on {Motion} {Groups}},''
  \emph{\BIBforeignlanguage{en}{The International Journal of Robotics
  Research}}, vol.~27, no. 11-12, pp. 1258--1273, Nov. 2008, publisher: SAGE
  Publications Ltd STM.

\bibitem{chirikjian_gaussian_2014}
G.~Chirikjian and M.~Kobilarov, ``\BIBforeignlanguage{en}{Gaussian
  approximation of non-linear measurement models on {Lie} groups},'' in
  \emph{\BIBforeignlanguage{en}{53rd {IEEE} {Conference} on {Decision} and
  {Control}}}.\hskip 1em plus 0.5em minus 0.4em\relax Los Angeles, CA, USA:
  IEEE, Dec. 2014, pp. 6401--6406.

\bibitem{chauchat_invariant_2018}
P.~Chauchat, A.~Barrau, and S.~Bonnabel, ``\BIBforeignlanguage{en}{Invariant
  smoothing on {Lie} {Groups}},'' in \emph{\BIBforeignlanguage{en}{2018
  {IEEE}/{RSJ} {International} {Conference} on {Intelligent} {Robots} and
  {Systems} ({IROS})}}.\hskip 1em plus 0.5em minus 0.4em\relax Madrid: IEEE,
  Oct. 2018, pp. 1703--1710.

\bibitem{mangelson_characterizing_2020}
J.~G. Mangelson, M.~Ghaffari, R.~Vasudevan, and R.~M. Eustice,
  ``\BIBforeignlanguage{en}{Characterizing the {Uncertainty} of {Jointly}
  {Distributed} {Poses} in the {Lie} {Algebra}},''
  \emph{\BIBforeignlanguage{en}{IEEE Transactions on Robotics}}, vol.~36,
  no.~5, pp. 1371--1388, Oct. 2020.

\bibitem{hartley_contact-aided_2020}
R.~Hartley, M.~Ghaffari, R.~M. Eustice, and J.~W. Grizzle,
  ``\BIBforeignlanguage{en}{Contact-aided invariant extended {Kalman} filtering
  for robot state estimation},'' \emph{\BIBforeignlanguage{en}{The
  International Journal of Robotics Research}}, vol.~39, no.~4, pp. 402--430,
  Mar. 2020, publisher: SAGE Publications Ltd STM.

\bibitem{saccon_optimal_2013}
A.~Saccon, J.~Hauser, and A.~P. Aguiar, ``Optimal {Control} on {Lie} {Groups}:
  {The} {Projection} {Operator} {Approach},'' \emph{IEEE Transactions on
  Automatic Control}, vol.~58, no.~9, pp. 2230--2245, Sep. 2013.

\bibitem{kobilarov_discrete_2011}
M.~B. Kobilarov and J.~E. Marsden, ``Discrete {Geometric} {Optimal} {Control}
  on {Lie} {Groups},'' \emph{IEEE Transactions on Robotics}, vol.~27, no.~4,
  pp. 641--655, Aug. 2011.

\bibitem{oseledets_tensor-train_2024}
I.~Oseledets, ``Tensor-{Train} {Toolbox} (ttpy),'' Jan. 2024, original-date:
  2012-08-21T18:22:27Z.

\bibitem{dagum_openmp_1998}
L.~Dagum and R.~Menon, ``{OpenMP}: an industry standard {API} for shared-memory
  programming,'' \emph{IEEE Computational Science and Engineering}, vol.~5,
  no.~1, pp. 46--55, Jan. 1998, conference Name: IEEE Computational Science and
  Engineering.

\bibitem{ravichandar_recent_2020}
H.~Ravichandar, A.~S. Polydoros, S.~Chernova, and A.~Billard, ``Recent
  {Advances} in {Robot} {Learning} from {Demonstration},'' \emph{Annual Review
  of Control, Robotics, and Autonomous Systems}, vol.~3, no.~1, pp. 297--330,
  2020, \_eprint: https://doi.org/10.1146/annurev-control-100819-063206.

\bibitem{wu_prim-lafd_2023}
Z.~Wu, W.~Lian, C.~Wang, M.~Li, S.~Schaal, and M.~Tomizuka, ``Prim-{LAfD}: {A}
  {Framework} to {Learn} and {Adapt} {Primitive}-{Based} {Skills} from
  {Demonstrations} for {Insertion} {Tasks},'' \emph{IFAC-PapersOnLine},
  vol.~56, no.~2, pp. 4120--4125, Jan. 2023.

\bibitem{zhang_vision-based_2023}
K.~Zhang, C.~Wang, H.~Chen, J.~Pan, M.~Y. Wang, and W.~Zhang,
  ``\BIBforeignlanguage{en}{Vision-based {Six}-{Dimensional} {Peg}-in-{Hole}
  for {Practical} {Connector} {Insertion}},'' in
  \emph{\BIBforeignlanguage{en}{2023 {IEEE} {International} {Conference} on
  {Robotics} and {Automation} ({ICRA})}}.\hskip 1em plus 0.5em minus
  0.4em\relax London, United Kingdom: IEEE, May 2023, pp. 1771--1777.

\bibitem{wen_you_2022}
B.~Wen, W.~Lian, K.~Bekris, and S.~Schaal, ``\BIBforeignlanguage{en}{You {Only}
  {Demonstrate} {Once}: {Category}-{Level} {Manipulation} from {Single}
  {Visual} {Demonstration}},'' in \emph{\BIBforeignlanguage{en}{Robotics:
  {Science} and {Systems} {XVIII}}}.\hskip 1em plus 0.5em minus 0.4em\relax
  Robotics: Science and Systems Foundation, Jun. 2022.

\bibitem{englert_learning_2018}
P.~Englert and M.~Toussaint, ``Learning manipulation skills from a single
  demonstration,'' \emph{The International Journal of Robotics Research},
  vol.~37, no.~1, pp. 137--154, Jan. 2018, publisher: SAGE Publications Ltd
  STM.

\bibitem{saveriano_dynamic_2023}
M.~Saveriano, F.~J. Abu-Dakka, A.~Kramberger, and L.~Peternel,
  ``\BIBforeignlanguage{en}{Dynamic movement primitives in robotics: {A}
  tutorial survey},'' \emph{\BIBforeignlanguage{en}{The International Journal
  of Robotics Research}}, vol.~42, no.~13, pp. 1133--1184, Nov. 2023,
  publisher: SAGE Publications Ltd STM.

\bibitem{jha_imitation_2022}
D.~K. Jha, D.~Romeres, W.~Yerazunis, and D.~Nikovski,
  ``\BIBforeignlanguage{en}{Imitation and {Supervised} {Learning} of
  {Compliance} for {Robotic} {Assembly}},'' in
  \emph{\BIBforeignlanguage{en}{2022 {European} {Control} {Conference}
  ({ECC})}}.\hskip 1em plus 0.5em minus 0.4em\relax London, United Kingdom:
  IEEE, Jul. 2022, pp. 1882--1889.

\bibitem{davchev_residual_2022}
T.~Davchev, K.~S. Luck, M.~Burke, F.~Meier, S.~Schaal, and S.~Ramamoorthy,
  ``Residual {Learning} {From} {Demonstration}: {Adapting} {DMPs} for
  {Contact}-{Rich} {Manipulation},'' \emph{IEEE Robotics and Automation
  Letters}, vol.~7, no.~2, pp. 4488--4495, Apr. 2022.

\bibitem{luo_robust_2021}
J.~Luo*, O.~Sushkov*, R.~Pevceviciute*, W.~Lian, C.~Su, M.~Vecerik, N.~Ye,
  S.~Schaal, and J.~Scholz, ``\BIBforeignlanguage{en}{Robust {Multi}-{Modal}
  {Policies} for {Industrial} {Assembly} via {Reinforcement} {Learning} and
  {Demonstrations}: {A} {Large}-{Scale} {Study}},'' in
  \emph{\BIBforeignlanguage{en}{Robotics: {Science} and {Systems}
  {XVII}}}.\hskip 1em plus 0.5em minus 0.4em\relax Robotics: Science and
  Systems Foundation, Jul. 2021.

\bibitem{ahn_robotic_2023}
K.-H. Ahn, M.~Na, and J.-B. Song, ``Robotic assembly strategy via reinforcement
  learning based on force and visual information,'' \emph{Robotics and
  Autonomous Systems}, vol. 164, p. 104399, Jun. 2023.

\bibitem{guo_reinforcement_2023}
Y.~Guo, J.~Gao, Z.~Wu, C.~Shi, and J.~Chen,
  ``\BIBforeignlanguage{en}{Reinforcement learning with {Demonstrations} from
  {Mismatched} {Task} under {Sparse} {Reward}},'' in
  \emph{\BIBforeignlanguage{en}{Proceedings of {The} 6th {Conference} on
  {Robot} {Learning}}}.\hskip 1em plus 0.5em minus 0.4em\relax PMLR, Mar. 2023,
  pp. 1146--1156, iSSN: 2640-3498.

\bibitem{noauthor_arm_nodate}
``Arm {Control} {System} — support.rethinkrobotics.com.''

\bibitem{chen_ak_2022}
W.~Chen, R.~Khardon, and L.~Liu, ``\BIBforeignlanguage{en}{{AK}: {Attentive}
  {Kernel} for {Information} {Gathering}},'' in
  \emph{\BIBforeignlanguage{en}{Robotics: {Science} and {Systems}
  {XVIII}}}.\hskip 1em plus 0.5em minus 0.4em\relax Robotics: Science and
  Systems Foundation, Jun. 2022.

\bibitem{levine_guided_2013}
S.~Levine and V.~Koltun, ``\BIBforeignlanguage{en}{Guided {Policy} {Search}},''
  in \emph{\BIBforeignlanguage{en}{Proceedings of the 30th {International}
  {Conference} on {Machine} {Learning}}}.\hskip 1em plus 0.5em minus
  0.4em\relax PMLR, May 2013, pp. 1--9, iSSN: 1938-7228.

\bibitem{gelfand_calculus_2000}
I.~M. Gelfand, S.~V. Fomin, and R.~A. Silverman,
  \emph{\BIBforeignlanguage{en}{Calculus of {Variations}}}.\hskip 1em plus
  0.5em minus 0.4em\relax Courier Corporation, Jan. 2000.

\end{thebibliography}


\appendix

\noindent\textbf{Proof for Lemma~\ref{lemma:minimal_norm_uniformality}}
\begin{proof}
We just need to prove that the trajectory empirical distribution $c_s(x)$ is an uniform distribution when it minimizes $\int_{\mathcal{S}} c_s(x)^2 dx$. This can formulated as the following functional optimization problem:
\begin{align}
    p^*(x) & = \argmin_{p(x)} \int_{\mathcal{S}} p(x)^2 dx,  \text{ s.t. } \int_{\mathcal{S}} p(x) dx = 1 
\end{align}
To solve it, we first formulate the Lagrangian operator:
\begin{align}
    \mathcal{L}(p, \lambda) = \int_{\mathcal{S}} p(x)^2 dx - \lambda \cdot \left(\int_{\mathcal{S}} p(x) dx - 1\right)
\end{align}
The necessary condition for $p^*(x)$ to be an extreme is (Theorem 1, Page 43~\cite{gelfand_calculus_2000}):
\begin{align}
    \frac{\partial\mathcal{L}}{\partial p}(p^*, \lambda) & = 2 p^*(x) - \lambda = 0
\end{align} which gives us $p^*(x) = \frac{\lambda}{2}$. By substituting this equation back to the equality constraint we have:
\begin{align}
    \int_{\mathcal{S}} p^*(x) dx & = \frac{\lambda}{2} \int_{\mathcal{S}} 1\cdot dx = \frac{\lambda}{2} \vert\mathcal{S}\vert = 1
\end{align}
Therefore $\lambda = \frac{2}{\vert\mathcal{S}\vert}$, and we have:
\begin{align}
    p^*(x) = \frac{\lambda}{2} = \frac{1}{\vert\mathcal{S}\vert}
\end{align} which is the probablity density function of a uniform distribution. To show that $p^*(x)$ as an extreme is a minimum instead of a maximum, we just need to find a distribution that has larger norm than $p^*(x)$. To do so, we define a distribution $p^\prime(x)$ that has value $\frac{1}{2\vert\mathcal{S}\vert}$ for half the search space $\mathcal{S}$, and has value of $\frac{3}{2\vert\mathcal{S}\vert}$. It's easy to show that $\Vert p^\prime(x)\Vert>\Vert p^*(x)\Vert$, thus $p^*(x)$ is the global minimum, which completes the proof.
\end{proof}

\noindent\textbf{Proof for Lemma~\ref{lemma:gateaux_kernel}}
\begin{proof}
    Based on the definition of Gateaux derivative, we have:
    \begin{align}
        & D \mathcal{E}_{\theta}(s(t),p(x)) \cdot z(t) \nonumber \\
        & = \lim_{\epsilon\rightarrow 0} \frac{d}{d\epsilon} \mathcal{E}_{\theta}(s(t) + \epsilon\cdot z(t), p(x)) \nonumber \\
        & = - \frac{2}{T} \int_0^T \lim_{\epsilon\rightarrow 0} \left[ \frac{d}{d\epsilon} p\Big(s(t) + \epsilon z(t)\Big) \right] dt + \frac{1}{T^2} {\int_{0}^{T}\int_{0}^{T}} \nonumber \\
        & \quad\quad \lim_{\epsilon\rightarrow 0} \left[ \frac{d}{d\epsilon} \phi_\theta\Big(s(t_1){+}\epsilon z(t_1), s(t_2){+}\epsilon z(t_2)\Big) \right] dt_1 dt_2 \nonumber \\
        & = - \frac{2}{T} \int_0^T \lim_{\epsilon\rightarrow 0} \left[ \frac{d}{d\epsilon} p\Big(s(t) + \epsilon z(t)\Big) \right] dt + \frac{1}{T^2} {\int_{0}^{T}\int_{0}^{T}} \nonumber \\
        & \quad\quad\quad \lim_{\epsilon\rightarrow 0} \left[ \frac{d}{d\epsilon} \phi_\theta\Big(s(t_1){+}\epsilon z(t_1), s(t_2){+}\epsilon z(t_2)\Big) \right] dt_1 dt_2 \nonumber \\
        & = - \frac{2}{T} \int_0^T \frac{d}{ds(t)} p\Big(s(t)\Big) \cdot z(t) dt \nonumber \\
        & \quad + \frac{1}{T^2} {\int_{0}^{T}\int_{0}^{T}} \left[ \frac{d}{ds(t_1)} \phi_\theta\Big(s(t_1), s(t_2)\Big) \right] \cdot z(t_1) \nonumber \\
        & \quad\quad\quad\quad\quad\quad + \left[ \frac{d}{ds(t_2)} \phi_\theta\Big(s(t_1), s(t_2)\Big) \right] \cdot z(t_2) dt_1 dt_2. 
    \end{align} Since the Gaussian kernel function $\phi_\theta(\cdot,\cdot)$ is symmetric and stationary, we have:
    \begin{align}
        & {\int_{0}^{T}\int_{0}^{T}} \left[ \frac{d}{ds(t_1)} \phi_\theta\Big(s(t_1), s(t_2)\Big) \right] \cdot z(t_1) dt_1 dt_2 \nonumber \\
        & = {\int_{0}^{T}\int_{0}^{T}} \left[ \frac{d}{ds(t_2)} \phi_\theta\Big(s(t_1), s(t_2)\Big) \right] \cdot z(t_2) dt_1 dt_2.
    \end{align} Therefore, we have:
    \begin{align}
        & D \mathcal{E}_{\theta}(s(t),p(x)) \cdot z(t) \nonumber \\
        & = - \frac{2}{T} \int_0^T \frac{d}{ds(t)} p\Big(s(t)\Big) \cdot z(t) dt \nonumber \\
        & \quad + \frac{2}{T^2} {\int_{0}^{T}\int_{0}^{T}} \left[ \frac{d}{ds(t_1)} \phi_\theta\Big(s(t_1), s(t_2)\Big) \right] \cdot z(t_1) dt_1 dt_2 \nonumber \\
        & = \int_0^T \Bigg[-\frac{2}{T} \frac{d}{ds(t)} p\Big(s(t)\Big) + \nonumber \\
        & \quad\quad\quad\quad\quad \frac{2}{T^2} \int_0^T \frac{d}{ds(t)} \phi_\theta\Big(s(t), s(\tau) \Big) d\tau \Bigg] \cdot z(t) dt,
    \end{align} which completes the proof.
\end{proof}

\end{document}